\documentclass{article}

 \usepackage[main, final]{neurips_2026}

\usepackage[utf8]{inputenc} 
\usepackage[T1]{fontenc}    
\usepackage{hyperref}       
\usepackage{url}            
\usepackage{booktabs}       
\usepackage{amsfonts}       
\usepackage{nicefrac}       
\usepackage{microtype}      
\usepackage{xcolor}         
\usepackage{siunitx}

\usepackage{graphicx}
\hypersetup{hidelinks}
\usepackage{multirow} 
\usepackage{tikz-cd}
\usepackage{algorithm}
\usepackage{algpseudocodex}
\usepackage[font=small,labelfont=rm]{subcaption}
\usepackage[font=small,margin=10pt]{caption}
\usepackage{amsmath,amssymb,amsfonts}
\usepackage{bm,mathtools}
\usepackage{amsthm}

\theoremstyle{plain}
\newtheorem{theorem}{Theorem}[section] 
\newtheorem{proposition}[theorem]{Proposition}

\newtheorem{lemma}[theorem]{Lemma}
\newtheorem{corollary}[theorem]{Corollary}
\newtheorem{definition}{Definition}

\newtheorem{remark}{Remark}

\def\expect{{\mathbb{E}}}

\def\Prob{\mathbb{P}}

\def\Pr{\text{Pr}}

\newcommand{\randvar}[1]{\mathbf{#1}}

\def\wasserstein{{\mathbb{W}}}

\def\indicator{{\mathbf{1}}}

\def\det{{\mathrm{det}}}

\newcommand{\norm}[1]{\left\lVert#1\right\rVert}

\def\realNum{{\mathbb{R}}}

\def\natNum{\mathbb{N}}

\def\sB{{\mathcal{B}}}

\def\sH{{\mathcal{H}}}

\def\sL{{\mathcal{L}}}

\def\sN{{\mathcal{N}}}

\def\sP{{\mathcal{P}}}

\def\sR{{\mathcal{R}}}

\def\sV{{\mathcal{V}}}

\def\sX{{\mathcal{X}}}
\def\sY{{\mathcal{Y}}}

\def\bsR{{\bm{\mathcal{R}}}}
\def\bsC{{\bm{\mathcal{C}}}}

\def\bsV{{\bm{\mathcal{V}}}}

\def\vlambda{{{\lambda}}}

\def\vpi{{{\pi}}}

\def\vp{{{p}}}

\newcommand{\evlambda}[1]{{\vlambda^{(#1)}}}

\newcommand{\evpi}[1]{{\vpi^{(#1)}}}

\newcommand{\evp}[1]{{\vp^{(#1)}}}

\DeclarePairedDelimiter\floor{\lfloor}{\rfloor}

\newcommand{\smoothed}[1]{\Tilde{#1}}

\newcommand{\ubar}[1]{\underline{#1}}

\newcommand{\EF}[1]{\textcolor{orange}{[EF: #1]}}

\newcommand{\LL}[1]{\textcolor{violet}{[LL: #1]}}

\newif\ifhighlightchanges
\highlightchangesfalse

\newcommand{\review}[1]{%
    \ifhighlightchanges
        \textcolor{blue}{#1}%
    \else
            #1%
    \fi
}

\newcommand{\reviewtheory}[1]{%
    \ifhighlightchanges
        \textcolor{red}{#1}%
    \else
            #1%
    \fi
}

\title{Clustered Randomized Smoothing \\ for Stochastic Prediction Functions}

\author{%
    \textbf{Eduardo Figueiredo}\thanks{Equal contribution.}\kern0.4em$^{1}$ \quad
    \textbf{Frederik Mathiesen}\footnotemark[1]\kern0.4em$^{2}$ \quad \textbf{Julian Schumann}$^{1}$ \\ \textbf{Jens Kober}$^{3}$ \quad \textbf{Arkady Zgonnikov}$^{1}$ \quad \textbf{Luca Laurenti}$^{1,4}$\\
    $^{1}$Delft University of Technology \quad $^{2}$University of Oxford \quad $^{3}$University of Stuttgart \quad $^{4}$AI4I\\
    \texttt{e.figueiredo@tudelft.nl}, \; \texttt{frederik.mathiesen@cs.ox.ac.uk}
}

\begin{document}

\maketitle

\begin{abstract}

Modern stochastic predictors can model rich, multi-modal outcome distributions. However, this expressive power comes with challenges in ensuring robust predictions $-$ a critical requirement in safety-critical domains. Randomized smoothing is a leading technique for improving robustness, particularly against adversarial perturbations. Yet, in stochastic multi-modal regression settings, randomized smoothing often fails due to mode collapse, yielding averaged predictions that do not reflect the underlying distribution. To address this limitation, we propose clustered $\alpha$-smoothing, a framework that (1) partitions noisy samples using an arbitrary clustering algorithm, (2) applies $\alpha$-smoothing locally within each cluster, and (3) combines the resulting predictions into a mixture distribution. By interpreting the smoothing distribution as a mixture of $\alpha$-smoothers, we derive a lower bound on the probability that the smoothed prediction lies within a union of compact regions corresponding to distinct modes. We empirically evaluate our framework on two benchmarks, demonstrating substantial improvements over state-of-the-art methods. In stochastic trajectory prediction on a driving simulator dataset, our approach achieves, on average, a \review{$27\%$} lower Wasserstein distance to the ground-truth distribution compared to $\alpha$-smoothing. In quadrotor control, where modes correspond to distinct feasible paths to a target, our method reduces the collision rate by $90\%$ relative to the state-of-the-art randomized smoothing.
\end{abstract}

\section{Introduction}
Stochastic predictors, such as variational autoencoders \cite{diederik2019introduction}, normalizing flows \cite{Papamakarios2021normalizing}, or Bayesian neural networks \cite{DBLP:conf/iclr/WuNMTHG19}, are capable of modeling rich, multi-modal outcome distributions \cite{lu2020universal, foong2020expressiveness}. However, this expressive power comes with challenges in ensuring robust predictions -- a fundamental requirement in safety-critical applications such as electrical grids, healthcare, robotics, or autonomous driving, where this type of brittle behavior can have catastrophic consequences~\cite{deng2016false, michelmore2020uncertainty, meyers2023safety, schumann2025realistic}. A particular concern is vulnerability to adversarial attacks~\cite{goodfellow2014explaining, madry2017towards, salman2019provably, qin2022boosting, chen2024diffusion}, in which carefully crafted, small perturbations of the input cause significant deviations in the network's output.
To address this issue, a wide range of methods has been developed for formally verifying desired input–output properties of neural networks, including mixed-integer programming~\cite{katz2017reluplex}, abstract interpretation~\cite{gehr2018ai2}, linear relaxations with branch-and-bound~\cite{zhang2018efficient, xu2020automatic}, satisfiability modulo theory solvers~\cite{wu2024marabou, duong2025neuralsat}, and reachability analysis~\cite{tran2020nnv, adams2023bnn}. Despite their strong guarantees, these methods are often computationally expensive, which limits their practical applicability.

\emph{Randomized smoothing} is a scalable alternative for certifying robustness \cite{liu2018towards, lecuyer2019certified, cohen2019certified}. Given a base predictor, randomized smoothing injects noise into the input and averages the resulting predictions, yielding a smoothed predictor with certified robustness guarantees. While most prior work focuses on classification \cite{liu2018towards, lecuyer2019certified, cohen2019certified}, recent efforts extend these ideas to regression  \cite{chiang2020detection,rekavandi2024rs,rekavandi2024certified,schumann2026evaluating}. In this setting, given a base predictor $h : \realNum^d \to \realNum^q$ and a Gaussian noise $\bm{\varepsilon} \sim \sN(\bm{0}, \sigma^2 I)$ with $\sigma^2 > 0$, the smoothed predictor is commonly defined as \review{(some variant of)} $\sH(x) = \expect_{\bm{\varepsilon}} \left[ h(x + \bm{\varepsilon}) \right]$.
However, this approach is fundamentally limited for stochastic neural network predictors. In these models, predictions are often inherently multi-modal \cite{diederik2019introduction, Papamakarios2021normalizing, DBLP:conf/iclr/WuNMTHG19}, yet randomized smoothing collapses this structure by averaging over noisy \review{outputs}, yielding a unimodal and potentially uninformative estimate (Figure~\ref{fig:trajectory-prediction-example-intro}). This limitation is particularly problematic in safety-critical applications where multi-modality is essential, such as trajectory prediction~\cite{nayakanti_wayformer_2023,meszaros_trajflow_2024,bae2024singulartrajectory,schumann2025realistic} 
or failure prediction~\cite{rivas2022predictions, zhuang20223prognostic}.

\begin{figure*}
    \centering
    \begin{subfigure}{0.3\textwidth}
        \centering
        \includegraphics[width=0.95\textwidth, page=1]{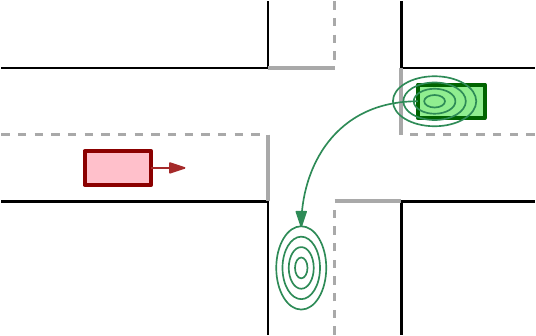}
        \caption{base predictor}
    \end{subfigure}%
    \hfill
    \begin{subfigure}{0.3\textwidth}
        \centering
        \includegraphics[width=0.95\textwidth, page=2]{Figures/trajectory_prediction_lgap.pdf}
        \caption{randomized smoothing}
    \end{subfigure}%
    \hfill
    \begin{subfigure}{0.3\textwidth}
        \centering
        \includegraphics[width=0.95\textwidth, page=3]{Figures/trajectory_prediction_lgap.pdf}
        \caption{our approach}
    \end{subfigure}
    \caption{Conceptual illustration of mode collapse in randomized smoothing. In (a), an ego vehicle (red) is predicting the trajectory of another vehicle (green) at a cross junction. The stochastic predictor outputs a distribution with two modes: \textit{turning left} before the ego vehicle passes the intersection, or \textit{waiting} for it to pass before initiating the turn. 
    In (b), current randomized smoothing methods (e.g.,~\cite{rekavandi2024certified}) would produce an averaged prediction, failing to capture any of the two modes. 
    Instead, in (c), our approach is to smooth each mode individually and combine as a mixture model, which preserves the modes corresponding to the two tactical behaviors. Sets $\sR$ represent the sets that can be certified with each method. 
    }
    \label{fig:trajectory-prediction-example-intro}
    \vspace{-0.5cm}
\end{figure*}

To address this limitation and capture the multi-modal behavior of the base predictor, we propose \emph{clustered $\alpha$-smoothing}. Given an input $x$ and a set of noise samples $\{\varepsilon_1, \ldots, \varepsilon_N\}$, our method (i) clusters the prediction samples $\{h_\randvar{w}(x+\varepsilon_1), \ldots, h_\randvar{w}(x+\varepsilon_N)\}$, (ii) trims $\alpha$ outlier samples per cluster, and (iii) smooths locally within each cluster. The clusters are then combined according to a mixture distribution. By combining the modes/clusters as mixture model rather than averaging, we avoid mode collapse, and thus, the limitation of the previous state-of-the-art. Further, the per-mode $\alpha$-trimming allows greater control over outliers.
We prove by classical probability theory results and statistical methods that the output of the smoothed predictor belongs to a coverage region $\sR$ with high probability for all inputs in a radius $r$. The guaranteed coverage region $\sR$ is provided as the union of the cluster-wise coverage regions $\sR_m$, \review{which are regions containing the samples of each cluster with high probability (see Figure \ref{fig:trajectory-prediction-example-intro})}, and the probability of the prediction belonging to the coverage region is provided both individually and jointly for all clusters. Our framework is agnostic to the clustering algorithm and provides flexibility in how to determine the number of clusters and sample assignment. 
We demonstrate the framework on two benchmarks: (1) stochastic multi-modal trajectory prediction in traffic (Figure~\ref{fig:trajectory-prediction-example-intro}), a key prerequisite for safe and efficient automated driving, and (2) robustification of a multi-modal deep RL quadrotor controller, where the framework helps improve robustness while preserving multiple feasible navigation strategies around obstacles.

In summary, our main contributions are:
(i) we introduce a framework for randomized smoothing for multi-modal stochastic predictors via clustering and $\alpha$-trimming, (ii) we prove that the output of the clustered $\alpha$-smoothed predictor with high probability falls within a small coverage region per cluster for all inputs in a radius, and (iii) we demonstrate the effectiveness of our framework in benchmarks from stochastic trajectory prediction in traffic and a multi-modal RL controller for a quadrotor, showing improved performance compared to state-of-the-art.




\section{Preliminaries on randomized smoothing}\label{section:preliminaries}
\paragraph*{Notation.}
For a vector $x \in \realNum^d$,  we denote by $x^{(i)}$ its $i$-element. Given $N\in\natNum_{>0}$, the set $\{1,\dots,N\}$ is denoted by $[N]$. $I$ denotes the identity matrix. The Pontryagin difference between sets $A, B$ is denoted by $A \ominus B$. For a set $\sX \subseteq \realNum^d$, the indicator function for $\sX$  is denoted as $\indicator_{\sX}(x) \coloneqq 1 \text{ if } x \in \sX \text{; otherwise } 0$.
Given a Borel measurable space $\sX \subseteq \realNum^d$, we denote by $\sB(\sX)$ the Borel sigma algebra over $\sX$ and by $\sP(\sX)$ the set of probability distributions on $\sX$.
For a random variable $\randvar{x}$ taking values in $\sX$, $\randvar{x} \sim \Prob_\randvar{x} \in \sP(\sX)$ represents the probability measure associated to $\randvar{x}$.
Let $f_{n, k}(p) = \binom{n}{k}p^k(1-p)^{n-k}$ denote throughout the paper the probability mass function of a binomial random variable with parameters $n$ and $k$. For a vector $x \in \sX$, the $k$-th order statistic of $x$ is denoted by $x_{(k)}$.
Further, we define a stochastic predictor as follows.

\begin{definition}[Stochastic Predictor]
    Consider a deterministic function $h_w : \sX \to \sY$ with $\sX \subseteq \realNum^d$ and $\sY \subset \realNum^q$, where $w \in \realNum^m$ is a parameter vector. Then, for $x\in \sX$ and a possibly input dependent random vector $\randvar{w} \sim \Prob_{\randvar{w}\mid x}$ on $\realNum^m$, the stochastic prediction induced by $h_w$ and $\randvar{w}$ at $x$ is the random variable $h_\randvar{w}(x)$.
\end{definition}
Our definition of stochastic predictor is sufficiently general to encompass important stochastic neural models such as Bayesian neural networks (BNNs) \cite{DBLP:conf/iclr/WuNMTHG19}, variational autoencoders (VAEs) \cite{diederik2019introduction}, and normalizing flows \cite{Papamakarios2021normalizing}. 
In what follows, for simplicity, we assume that the distribution of $\randvar{w}$ does not depend on the input $x$. Accordingly, we write $\Prob_{\randvar{w}}$ instead of $\Prob_{\randvar{w}\mid x}$. However, our methods in Sections~\ref{section:randomized-smoothing} and \ref{section:algorithm} also apply to the case where $\randvar{w}$ depends on the input $x$. 

\paragraph*{Randomized smoothing.}
Although the literature on randomized smoothing has been extensively developed for classification problems \cite{cohen2019certified,li2019certified,yang2020randomized,levine2020wasserstein}, recent work has sparked interest in regression settings \cite{chiang2020detection,rekavandi2024rs,rekavandi2024certified}. Given a stochastic predictor $h_\randvar{w}$ and  $x\in\sX$, randomized smoothing can be viewed as the transformation $\sH(x) = \expect_{\bm{\varepsilon}}[h_\randvar{w}(x + \bm{\varepsilon})]$, i.e., the original stochastic function is replaced by its expectation w.r.t.\ noisy perturbations around $x$. In practice, computing $\sH(x)$ in closed-form is often infeasible~\cite{girard2002gaussian}, thus commonly replaced by its sampled approximation $\sH_N(x) = \frac{1}{N} \sum_{i=1}^{N} h_\randvar{w}(x + \varepsilon_i)$, where $\{ \varepsilon_1,\dots,\varepsilon_N \}$ are i.i.d.\ samples from the distribution of $\varepsilon$. 
However, as discussed in \cite{rekavandi2024certified}, such \review{a} smoother presents a major weakness: the smoothed function is vulnerable to outliers in the prediction, with a single sample being capable of pushing the smoothed output beyond the desired range $\sR$.
Therefore, \cite{rekavandi2024certified} proposes \review{an} $\alpha$-trimming smoothing technique, which allows trimming these outliers. The $\alpha$-smoothed predictor is defined as follows.
\begin{definition}[$\alpha$-smoothing \cite{rekavandi2024certified}]\label{def:alpha-smoothing}
    Let $h_\randvar{w}$ be a stochastic predictor with $\sX,\sY,$ being its input and output spaces and $x \in \sX$. Given $N\in\natNum_{>0}$ i.i.d.\ samples $\{ \varepsilon_1,\dots,\varepsilon_N \}$ from a distribution $\Prob_\varepsilon \in \sP(\sX \ominus \{x\})$, define the random vector $\sH(x) = \big( h_\randvar{w}(x+\varepsilon_i) \big)_{i\in [N]}$. Then, for $\alpha \in [0, \frac{1}{2})$, the $\alpha$-smoothed predictor is defined as
    \begin{equation}\label{eq:def-alpha-smoothing}
        \smoothed{\sH}_{N,\alpha}(x) = \frac{1}{N - 2\floor{\alpha N}} \sum_{i=1+\floor{\alpha N}}^{N - \floor{\alpha N}} \sH(x)_{(i)},
    \end{equation}
    where the $i$-th \review{order statistic} in $\sH(x)_{(i)}$ is taken element-wise.
\end{definition}

$\alpha$-smoothing allows tighter control over the probability that $\smoothed{\sH}_{N,\alpha}(x') \in \sR$ for any convex set $\sR\subseteq \sY$. However, when the predictor distribution is multi-modal, the above-mentioned technique inadvertently collapses its modes (Section \ref{section:randomized-smoothing}). To solve this issue, we propose  \emph{clustered $\alpha$-smoothing}.
\review{Before introducing our method, we should note that clustering has previously been applied successfully in randomized smoothing to robustify large language models \cite{su2025smoothguard, wang2026clucert} under the justification that adversarial examples are often located in small clusters adjacent to the dominant mode. However, unlike these previous works, the safety-critical settings considered in this manuscript require that we account for behavior of the predictor in those minor clusters and thus cannot simply prune them, as done  in \cite{su2025smoothguard, wang2026clucert}.}


\section{\texorpdfstring{Clustered $\alpha$-Smoothing}{Clustered alpha-smoothing}}\label{section:randomized-smoothing}
\begin{figure*}
    \centering
    \includegraphics[width=0.9\textwidth]{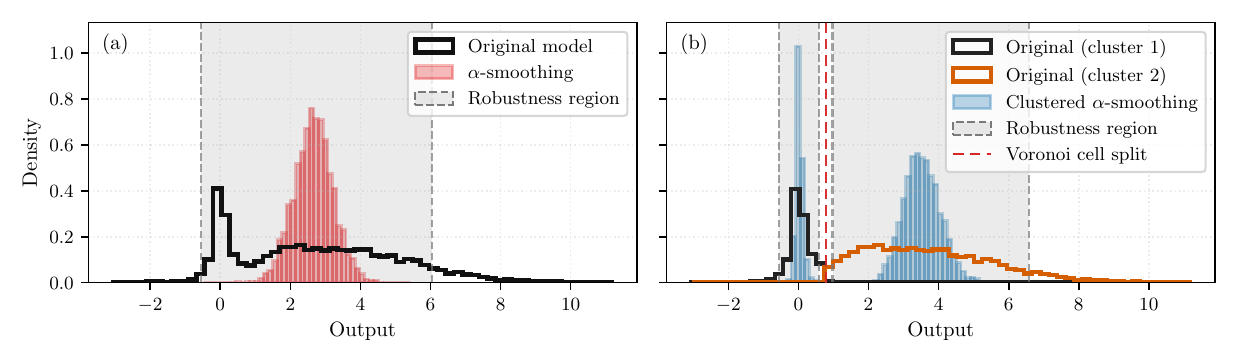}
    \vspace{-0.25cm}
    \caption{An example of $\alpha$-smoothing and clustered $\alpha$-smoothing at $x = 2$ of a stochastic base predictor $h_\randvar{w}(x) = \randvar{w}x$ where $\randvar{w} \sim \sum_{i = 1}^3 \pi_i \sN(\cdot \mid \mu_i, \sigma_i^2)$ is a mixture of three Gaussians with means $\mu = (1, 0, 2)$ and equal variance $\sigma_i^2 = 0.01$, and weights $\pi = (0.2, 0.2, 0.6)$. The histograms are obtained by Monte Carlo sampling the distributions of $h_\randvar{w}(x + \bm{\varepsilon})$, $\smoothed{\sH}_{N,\alpha,\{\sY\}}(x)$ (in (a), equivalent to $\alpha$-smoothing, see Remark \ref{rmk:compare-clustered-smoother-definition-rekavandi}), and our clustered $\alpha$-smoothing technique, $\smoothed{\sH}_{N,\alpha,\bsV}(x)$, in (b) constructed by Algorithm \ref{alg:prediction_certification}. The Voronoi cell split represents the partition $\bsV=\{ \sV_1, \sV_2 \}$ into two clusters, and the robustness regions in (b) denote sets that $\smoothed{\sH}_{N,\alpha,\bsV}(x)$ belongs to with high probability (see Algorithm \ref{alg:sets-construction}). }
    \label{fig:modes-synthetic}
    \vspace{-0.5cm}
\end{figure*}

In this section, we introduce a novel randomized smoothing method tailored to stochastic predictors, which we call \emph{clustered $\alpha$-smoothing}. The key idea is to  (i) cluster samples in the output space,  (ii) apply smoothing locally within each cluster, and 
(iii) combine the resulting predictors into a mixture distribution, with weights proportional to the size of each cluster.
By combining the locally-smoothed predictors as a mixture, as illustrated in Figure~\ref{fig:modes-synthetic}, we prevent the mode collapse observed in $\alpha$-smoothing~\cite{rekavandi2024certified}. 
To formally define clustered $\alpha$-smoothing, we let $\bsV = \{\sV_1, \dots, \sV_M\}$ be a partition of the output space $\sY$ into $M$ disjoint sets, each corresponding to a mode of the prediction. For now, we assume that $\bsV$ is given; in Section~\ref{section:algorithm}, we describe how to construct $\bsV$ via clustering so that it aligns with the modes of the predictor. 

\begin{definition}[Clustered $\alpha$-smoothing]\label{def:clustered-alpha-smoothing}
    Let $h_\randvar{w}$ be a stochastic predictor, $x \in \sX$ a point in its domain, and $\bsV = \{\sV_1, \dots, \sV_M\}$ a partition of $\sY$. Given $N\in\natNum_{>0}$ i.i.d.\ samples $\{ \varepsilon_1,\dots,\varepsilon_N \}$ from a distribution $\Prob_\varepsilon \in \sP(\sX \ominus \{x\})$, we define the set of \review{indices} $I_m \subset [N]$ as $I_m = \Big\{ i\in [N] : h_\randvar{w}(x+\varepsilon_i) \in \sV_m \Big\}$. Then, define  $\sH_{\sV_m}(x) = \big( h_\randvar{w}(x+\varepsilon_i) \big)_{i\in I_m}$ and, for $\alpha \in [0, \frac{1}{2})$ call
    \begin{equation}\label{eq:intra-cluster-alpha-trimming-def}
        \smoothed{\sH}_{N,\alpha, \sV_m}(x) = \frac{1}{|I_m| - 2\floor{\alpha |I_m|}} \sum_{i=1+\floor{\alpha |I_m|}}^{|I_m| - \floor{\alpha |I_m|}} \sH_{\sV_m}(x)_{(i)},
    \end{equation}
    where the $i$-th order statistic in $\sH_{\sV_m}(x)_{(i)}$ is taken element-wise\footnote{\reviewtheory{If $I_m = \emptyset$, \eqref{eq:intra-cluster-alpha-trimming-def} is undefined and we set $\smoothed{\sH}_{N,\alpha, \sV_m}(x)$ to an arbitrary fixed value. This is immaterial, as its weight $\evpi{m} = 0$ means that it is never selected in \eqref{eq:clustered-smoothing}.}}.
    Then, the clustered $\alpha$-smoothing is defined as follows
    \begin{equation}\label{eq:clustered-smoothing}
        \smoothed{\sH}_{N,\alpha, \bsV}(x) = \sum_{m=1}^M \indicator_{\randvar{z} = m}\smoothed{\sH}_{N,\alpha, \sV_m}(x)
    \end{equation}
    where $\randvar{z}$ is a categorical random variable with class weights $\evpi{m} = \frac{|I_m|}{N}$.
\end{definition}

Intuitively, the set of \review{indices} $I_m$ identifies to which partition cell each prediction sample $h_\randvar{w}(x+\varepsilon_i)$ belongs. Then $\smoothed{\sH}_{N,\alpha, \sV_m}$ applies the $\alpha$-trimming average to each partition cell individually. Finally, the \emph{clustered $\alpha$-smoother} $\smoothed{\sH}_{N,\alpha, \bsV}(x)$ is constructed by randomly selecting each of the $\smoothed{\sH}_{N,\alpha, \sV_m}$ with a probability proportional to how many samples are in each $\sV_m$. 
The parameter $\alpha$ in Definition~\ref{def:clustered-alpha-smoothing} refers to trimming a $2\alpha$ fraction of samples, i.e., treating them as outliers, \emph{per cluster}. When $\alpha=0$, $\smoothed{\sH}_{N,\alpha,\sV_m}(x)$ reduces to the mean smoothing, $\smoothed{\sH}_{N,0,\sV_m}(x) = \frac{1}{|I_m|}\sum_{i\in I_m} h_\randvar{w}(x+\varepsilon_i)$. When $\alpha \to \frac{1}{2}$, $\smoothed{\sH}_{N,\alpha,\sV_m}(x)$ is equivalent to the median-smoothing, $\smoothed{\sH}_{N,\alpha,\sV_m}(x) = \text{median}_{i\in I_m} \big( h_\randvar{w}(x+\varepsilon_i) \big)$.

\begin{remark}[Relation with $\alpha$-smoothing in \cite{rekavandi2024certified}]\label{rmk:compare-clustered-smoother-definition-rekavandi}
    When $\bsV = \{ \sV \} = \{ \sY \}$, then Definition~\ref{def:clustered-alpha-smoothing} is equivalent to the $\alpha$-smoothing in \eqref{eq:def-alpha-smoothing}.
    However, when distributions exhibit multi-modality or are not strongly symmetric, then the flexibility of our approach can lead to a substantial improvement of performance over \cite{rekavandi2024certified}, as we will show in Section~\ref{section:experiments}.
\end{remark}

\textbf{Robustness Guarantees.} 
A core advantage of randomized smoothing is the ability to guarantee that, for sufficiently small perturbations, the output remains within a desired set with high probability. In this subsection, we develop this type of guarantee for clustered $\alpha$-smoothing. Our approach, which results in Theorem \ref{thm:robustness-certification-clustered-smoothing} below, is to 
apply the Neyman-Pearson Lemma \cite{cohen2019certified} to each cluster/mode, which are then combined to obtain worst-case bounds for the smoothed predictor via a union-bound argument \review{while accounting for} the fact that $\bsV$ is a partition of \review{$\sY$}.
To apply the lemma, we follow the predominant convention in the randomized smoothing literature and adopt $\bm{\varepsilon} \sim \sN(0, \sigma^2 I)$ for $\sigma^2 > 0$, which requires $\sX = \realNum^d$.


\begin{theorem}[Robustness certification for clustered $\alpha$-smoothing]\label{thm:robustness-certification-clustered-smoothing}
    Let $\bsV = \left\{\sV_1, \dots, \sV_M \right\}$ be a partition of the output space $\sY$. Assume that $\check{p}_{\sV_m} \leq \Prob_{(\randvar{w}, \bm{\varepsilon})}\left(h_\randvar{w}(x + \bm{\varepsilon}) \in \sV_m \right) \leq \hat{p}_{\sV_m}$ for each $\sV_m$. Furthermore, let $\sL \subset [M]$ and define $\tilde{\sR} = \cup_{l \in \sL} \sR_l$ for \reviewtheory{axis-aligned hyperrectangles $\sR_l \subset \sV_l$}. Assume that $\Prob_{(\randvar{w}, \bm{\varepsilon})}\left(h_\randvar{w}(x + \bm{\varepsilon}) \in \sR_l \right) \geq \check{p}_{\sR_l}$ for each $\sR_l$. For $r \geq 0$, define for each $m \in [M]$ and $l \in \sL$
    \begin{equation}\label{eq:adjustment-given-radius}
        \ubar{p}_{\sV_m} = \Phi\left(\Phi^{-1}(\check{p}_{\sV_m}) - \frac{r}{\sigma}\right), \; \bar{p}_{\sV_m} = \Phi\left(\Phi^{-1}(\hat{p}_{\sV_m}) + \frac{r}{\sigma}\right), \; \text{ and } \;
        \ubar{p}_{\sR_l} = \Phi\left(\Phi^{-1}(\check{p}_{\sR_l}) - \frac{r}{\sigma}\right).
    \end{equation}
    Then, for any $\delta \in \realNum^d$ such that $\norm{\delta}_2 \leq r$, it holds that
    \begin{align}\label{eq:main-guarantee-clustered-smoothing}
        &\Prob_{(\randvar{w}, \bm{\varepsilon}_1, \ldots, \bm{\varepsilon}_N)} \left(\smoothed{\sH}_{N, \alpha, \bsV}(x + \delta) \in \tilde{\sR} \right) \nonumber \\
        &\qquad \geq \inf_{\substack{p_{\sV_m} \in [\ubar{p}_{\sV_m}, \bar{p}_{\sV_m}] \\ \sum_{m=1}^M p_{\sV_m}=1}} \sum_{l \in \sL} \sum_{s=1}^{N} \frac{s}{N} \sum_{j=\review{s-\floor{\alpha s}}}^{s} f_{s,j} \left( \frac{\ubar{p}_{\sR_l}}{p_{\sV_l}} \right) f_{N,s} \left(p_{\sV_l}\right).
    \end{align}
\end{theorem}
The proof of Theorem \ref{thm:robustness-certification-clustered-smoothing} can be found in Appendix~\ref{section:proofs}. For each partition region $\sV_m$, the interval $[\ubar{p}_{\sV_m}, \bar{p}_{\sV_m}]$ bounds the probability that $h_\randvar{w}(x + \delta + \bm{\varepsilon})\in \sV_m$ for any perturbation $\delta$ with norm $\norm{\delta}_2\leq r$. Analogously, $\ubar{p}_{\sR_l}$ represents a probability lower bound on the event $h_\randvar{w}(x + \delta + \bm{\varepsilon})\in \sR_l$. Consequently, Theorem~\ref{thm:robustness-certification-clustered-smoothing} combines these bounds to obtain a lower bound for the probability that $\smoothed{\sH}_{N, \alpha, \bsV}(x + \delta) \in \tilde{\sR}$. Figure \ref{fig:bound-comparison} shows how the various parameters affect the resulting bounds. It can be observed that a higher trimming parameter $\alpha$ leads to a significantly faster increase of the bound, as outlier output values are eliminated from the average predictor in \eqref{eq:intra-cluster-alpha-trimming-def}, while higher ratios $\frac{p_{\sR_l}}{p_{\sV_l}}$ -- i.e. larger coverage in the partition set -- naturally increases the bound.

\begin{figure*}
    \centering
    \vspace{-0.25cm}
    \includegraphics[width=1.0\textwidth]{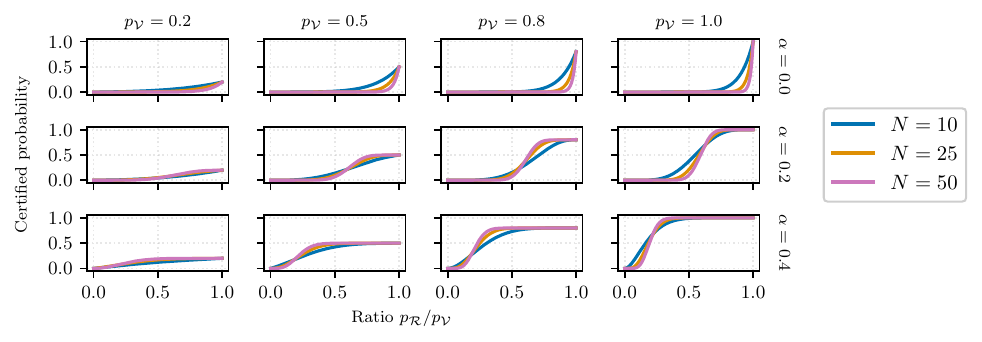}
    \caption{Probability bounds $\sum_{s=1}^{N} \frac{s}{N} \sum_{j=s-\floor{\alpha s}}^{s} f_{s,j} \left( \frac{p_{\sR}}{p_{\sV}} \right) f_{N,s} \left(p_{\sV}\right)$ from \eqref{eq:main-guarantee-per-region-clustered-smoothing} attributable to each mode for different combinations of number of smoothing samples $N$, trimming parameter $\alpha$, partition set probability $p_\sV$ and coverage ratio $p_\sR / p_\sV$. The right-most column (where $p_\sV=1$) is equivalent to the bound in \cite{rekavandi2024certified}, by noting that only the term $s=N$ in \eqref{eq:main-guarantee-per-region-clustered-smoothing} is non-zero when $\ubar{p}_\sV=p_\sV=\bar{p}_\sV=1$.}
    \label{fig:bound-comparison}
    \vspace{-0.5cm}
\end{figure*}

Before showing how to compute \eqref{eq:main-guarantee-clustered-smoothing} in practice, it is instructive to examine the properties of the case $\lvert \sL \rvert = 1$ first. In particular, the objective function in this case is independent from all other indices $[M] \setminus \sL$, and so, as we formally prove in \reviewtheory{Corollary~\ref{cor:per-region-closed-form}}, the bound in \eqref{eq:main-guarantee-clustered-smoothing} \reviewtheory{can be relaxed to the closed form (each factor at its worst case over $[\ubar{p}_{\sV_l}, \bar{p}_{\sV_l}]$):}
\begin{align}\label{eq:main-guarantee-per-region-clustered-smoothing}
    &\Prob_{(\randvar{w}, \bm{\varepsilon}_1, \ldots, \bm{\varepsilon}_N)} \left(\smoothed{\sH}_{N, \alpha, \bsV}(x + \delta) \in \sR_l\right) \nonumber \\
    &\qquad \geq \sum_{s=1}^{N} \frac{s}{N} \sum_{j=\review{s-\floor{\alpha s}}}^{s} f_{s,j}\left(\frac{\ubar{p}_{\sR_l}}{\bar{p}_{\sV_l}} \right) \min\left\{f_{N,s}\left(\ubar{p}_{\sV_l}\right), \; f_{N,s}\left(\bar{p}_{\sV_l}\right) \right\}.
\end{align}

\paragraph{Solving the general case.}
Unfortunately, the computational convenience of the case $|\sL|=1$ does not extend to the general case $|\sL|>1$, as in this setting \eqref{eq:main-guarantee-clustered-smoothing} consists of a non-convex optimization problem. Instead, to achieve a computationally tractable result, we show how the optimization problem in \eqref{eq:main-guarantee-clustered-smoothing} can be lower bounded by a linear program via \emph{anchor points} and an error that can be controlled by the user. This is presented in Proposition \ref{prop:prob_bound_with_anchor_points}. 

\begin{proposition}\label{prop:prob_bound_with_anchor_points}
   Assume that $\ubar{p}_{\sV_m}, \bar{p}_{\sV_m}$ for every $m \in [M]$ and $\ubar{p}_{\sR_l}$ for every $l \in \sL$  are given\reviewtheory{, with $0 \leq \ubar{p}_{\sR_l} < \ubar{p}_{\sV_l}$ and $\bar{p}_{\sV_l} < 1$ for every $l \in \sL$}.
   Define for each component $l \in \sL$ the anchor points $p_{\sV_l}^{(1)} < p_{\sV_l}^{(2)} < \cdots < p_{\sV_l}^{(K)}$ \reviewtheory{with $p_{\sV_l}^{(1)} = \ubar{p}_{\sV_l}$ and $p_{\sV_l}^{(K)} = \bar{p}_{\sV_l}$}.
    Then, it holds that
    \begin{align}\label{eq:bound-union-lp-formulation}
        &\inf_{\substack{p_{\sV_m} \in [\ubar{p}_{\sV_m}, \bar{p}_{\sV_m}] \\ \sum_{m=1}^M p_{\sV_m}=1}} \sum_{l \in \sL} \sum_{s=1}^{N} \frac{s}{N} \sum_{j=\review{s-\floor{\alpha s}}}^{s} f_{s,j} \left( \frac{\ubar{p}_{\sR_l}}{p_{\sV_l}} \right) f_{N,s} \left(p_{\sV_l}\right) \\
        &\qquad\geq \inf_{\bm{p},\, \bm{\lambda}_1, \ldots, \bm{\lambda}_{\lvert \sL \rvert}}
        \;\sum_{l \in \sL} \sum_{k=1}^K \lambda_l^{(k)} g_l(p_{\sV_l}^{(k)}) - \sum_{l \in \sL} \left[\frac{N}{1 - \bar{p}_{\sV_l}} + \frac{N}{\ubar{p}_{\sV_l} - \ubar{p}_{\sR_l}}\right] h_{l, K}\label{eq:union_objective} \\[0.75em]
        &\qquad\qquad\text{s.t.}\quad p_{\sV_m} \in [\ubar{p}_{\sV_m}, \bar{p}_{\sV_m}], \; \sum_{m=1}^M p_{\sV_m} = 1, \; \sum_{k=1}^K \lambda_l^{(k)} p_{\sV_l}^{(k)} = p_{\sV_l}, \; \bm{\lambda}_l \in \Delta_K,
    \end{align}
    where the variable $h_{l, K}$ and function $g_l$ are given by $h_{l, K} = \max_{p\in[\ubar{p}_{\sV_l}, \bar{p}_{\sV_l}]} \min_{k \in [K]} \lvert p - p_{\sV_l}^{(k)} \rvert$ and
$        g_l(p) = \sum_{s=1}^{N} \frac{s}{N} f_{N,s}(p) \sum_{j=\review{s-\floor{\alpha s}} }^{s} f_{s,j} \left( \frac{\ubar{p}_{\sR_l}}{p} \right). $
\end{proposition}

The intuition behind Proposition 
\ref{prop:prob_bound_with_anchor_points} is the following. We seek a tractable lower bound on the infimum in \eqref{eq:bound-union-lp-formulation}. This is generally intractable. Consequently, instead, we replace $p_{\sR_l}$ by $\ubar{p}_{\sR_l}$ to obtain a valid lower bound, and then approximate the convex 
envelope of each inner component using a finite set of evaluation points -- the \emph{anchor points} -- leading to a linear 
program in which the coupling constraint is handled natively. The bound is asymptotically exact as the number of anchor points 
$K$ increases, at the cost of increasing the size of the linear program. We call the second term of \eqref{eq:union_objective} the residual, which represents an adjustment needed due to the fact that we only approximate the convex envelope. Note that, as we show in Appendix~\ref{section:residual-analysis-prop-appendix}, one can always achieve an arbitrarily small residual value by increasing $K$. 

\begin{remark}[Extension to $L_\rho$-norm attack]\label{remark:extension-thm-lp-norm}
    Given $\rho \in [2, \infty) \cup \{ \infty \}$, the guarantee \eqref{eq:main-guarantee-clustered-smoothing} also holds for $\delta \in \realNum^d$ such that $\norm{\delta}_\rho \leq d^{\frac{1}{\rho}-\frac{1}{2}} r$ by a straightforward application of Holder's inequality \cite{rudin1987real}.
\end{remark}

\section{Algorithm}\label{section:algorithm}

In Theorem \ref{thm:robustness-certification-clustered-smoothing}, given a point $x\in\sX$, we need to compute probability bounds for the events $h_\randvar{w}(x+\varepsilon) \in \sV_m$ and $h_\randvar{w}(x+\varepsilon) \in \sR_l$, where $\bsV=\{ \sV_m \}_{m=1}^M$ is a partition of the output space $\sY$ into $M$ sets, and $\sR_l \subseteq \sV_l$ are given subsets. Given the (possibly) non-linearity of the underlying deterministic function $h_w$ on both the parameters $w$ and input $x$, computing those quantities explicitly is generally intractable. Thus, one needs to resort to sample-based methods, for which probability bounds are estimated with high-confidence.

\begin{proposition}[High-confidence probability bounds for sets]\label{prop:high-confidence-set-construction}
    Let $\bsV = \{\sV_m\}_{m = 1}^M$ be a partition of the space \review{$\sY$}, and $\sR_l \subseteq \sV_l$ be given subsets. Given a confidence parameter $\beta \in [0,1]$, and a number of samples $\bar{N} \in \natNum_{>0}$, let $Z_{\sV_m} = \sum_{i=1}^{\bar{N}} \indicator_{h_{w_i}(x+\varepsilon_i) \in \sV_m}$ and $Z_{\sR_m} = \sum_{i=1}^{\bar{N}} \indicator_{h_{w_i}(x+\varepsilon_i) \in \sR_m}$, where $(w_i, \varepsilon_i)$ are i.i.d.\ samples from $\Prob_{(\randvar{w}, \bm{\varepsilon})}$. If $\check{p}_{\sV_m}, \hat{p}_{\sV_m}, \check{p}_{\sR_m} \in [0,1]$ are such that\footnote{ If $Z_{\sV_m}=0$ (resp. $Z_{\sR_m}=0$), then $\check{p}_{\sV_m}=0$ (resp. $\check{p}_{\sR_m}=0$). Alternatively, if $Z_{\sV_m}=\bar{N}$, then $\hat{p}_{\sV_m}=1$. }
    \begin{equation}\label{eq:cp-partition-lower-bound}
        \frac{\beta}{3M} = \sum_{j=Z_{\sV_m}}^{\bar{N}} f_{\bar{N},j}\left( \check{p}_{\sV_m} \right), \;
        \frac{\beta}{3M} = \sum_{j=0}^{Z_{\sV_m}}f_{\bar{N},j}\left( \hat{p}_{\sV_m} \right), \; \text{ and } \;
        \frac{\beta}{3M} = \sum_{j=Z_{\sR_m}}^{\bar{N}} f_{\bar{N},j}\left( \check{p}_{\sR_m} \right),
    \end{equation}
    then, the assumptions in Theorem~\ref{thm:robustness-certification-clustered-smoothing} hold with confidence $1-\beta$.
\end{proposition}

Proposition \ref{prop:high-confidence-set-construction} comes from the Clopper-Pearson Lemma \cite{cohen2019certified} combined with a union-bound argument. It gives a natural way to construct the partition $\bsV$ and sets $\bsR$: after observing \reviewtheory{a first batch of samples}, one may place the sets $\sV_m$ and $\sR_m$ so that the number of outputs inside them (approximately) induces the desired probability level\reviewtheory{, and then apply Proposition~\ref{prop:high-confidence-set-construction} to fresh samples}. This reasoning is the core principle of Algorithm \ref{alg:sets-construction}.

\begin{algorithm}
\caption{Construction of partition $\bsV=\big\{ \sV_1,\dots,\sV_M \big\}$ and sets $\bsR=\big\{ \sR_1,\dots,\sR_M \big\}$ from samples such that $\sR_i \subset \sV_i$ and $\Prob(h_{\randvar{w}}(x + \varepsilon) \in \sR_i) / \Prob(h_{\randvar{w}}(x + \varepsilon) \in \sV_i) \approx p$.}
\label{alg:sets-construction}
\begin{algorithmic}[1]
\Function{ConstructCoverageSets}{$x$, $h_{\randvar{w}}$, $p$, $N$}
    \State $\mathsf{Outputs} \gets \emptyset$
    
    \For{$i \gets 1, \ldots, N$}
        \State Draw $\varepsilon_i \sim \mathcal{N}(0,\sigma^2 I)$ and 
        $w_{i} \sim \Prob_{\randvar{w}}$
    
        \State Compute $h_{w_i}(x+\varepsilon_i)$ and append to $\mathsf{Outputs}$
    \EndFor
    
    \State Cluster the elements of $\mathsf{Outputs}$ \Comment{e.g., using DBSCAN, DP-means, UMAP.}
    
    \State Denote the resulting clusters by $C_1, \dots, C_M$
    
    \For{$m \gets 1, \ldots, M$}
        \State $C_m^{(p)} \gets$ Select approximately a fraction $p$ of the points in $C_m$ \label{line:coverage_fraction}
        
        \State $\sR_m \gets \Call{Envelope}{C_m^{(p)}}$ \Comment{axis-aligned bounding hyperrectangle.}
    \EndFor

    \State \textbf{assert} $\sR_i \cap \sR_j = \emptyset$ for all $i \neq j$

    \State $\bsV \gets$ generalized Voronoi partition associated with $\bsR$ \label{line:gen_voronoi_partition}
    
    \State \Return partition $\bsV$ and coverage regions $\bsR$
\EndFunction
\end{algorithmic}
\end{algorithm}

Given the user's desired coverage level $p$, Algorithm \ref{alg:sets-construction} generates the sets $\sR_m$ to cover a fraction $p$ of the samples within $\sV_m$ (Line~\ref{line:coverage_fraction})\reviewtheory{, each enclosed by an axis-aligned hyperrectangle}, as required by Theorem~\ref{thm:robustness-certification-clustered-smoothing}. This parameter can be used to calibrate the trade-off between the size of the coverage sets and their certified probability. 
Next, by choosing the partition of the output space as a generalized Voronoi partition (closest \emph{region} instead of point, Line~\ref{line:gen_voronoi_partition}; see Appendix~\ref{section:generalized_voronoi} for details) \cite{milenkovic1993robust}, we enforce that the regions are contained in their partition, which is a requirement for Theorem~\ref{thm:robustness-certification-clustered-smoothing}. 
Although computing the (generalized) Voronoi partition associated with a set of regions $\bsR$ is hard (see Algorithm~\ref{alg:prediction_certification}), we only need to check whether a point belongs to $\sV_m$, which is equivalent to checking if $\sR_m$ is the closest region to this point. Thus, in practice, $\bsV$ is implicitly defined by $\bsR$. With partition $\bsV$, we can compute the robustness bounds introduced in Theorem~\ref{thm:robustness-certification-clustered-smoothing} for the clustered $\alpha$-smoother $\smoothed{\sH}_{N,\alpha,\bsV}$, as presented in Algorithm~\ref{alg:prediction_certification}. 
Appendix~\ref{section:experiment_details} expands on the parameters of Algorithm~\ref{alg:prediction_certification} and their impact on the coverage regions and certified probabilities.

\begin{algorithm}
\caption{Robustness certification of $\smoothed{\sH}_{N, \alpha, \bsV}$ at $x \in \sX$\protect\footnotemark.}
\label{alg:prediction_certification}
\begin{algorithmic}[1]
\Require Point $x \in \mathbb{R}^d$, stochastic function $h_{\randvar{w}}$, smoothing parameters \review{$\alpha \in [0,1/2)$} and $N \in \natNum_{>0}$, coverage level $p\in[0,1]$, robustness radius $r\geq 0$, number of samples $\bar N \in \natNum_{>0}$, and confidence level $\beta\in (0,1)$.
\Ensure With confidence $1-\beta$, $\smoothed{\sH}_{N,\alpha,\bsV}(x+\delta) \in \tilde{\sR}$ with probability given by \eqref{eq:main-guarantee-clustered-smoothing} for any $\delta \in \realNum^d$ such that $\norm{\delta}_2 \leq r$.

\State $\bsV, \bsR \gets \Call{ConstructCoverageSets}{x, h_{\randvar{w}}, p, \bar{N}}$

\State Construct $\smoothed{\sH}_{N,\alpha,\bsV}(x)$

\State Compute $\check{p}_{\sV_m}, \hat{p}_{\sV_m}, \check{p}_{\sR_m} \in [0,1]$ with Prop. \ref{prop:high-confidence-set-construction} for $m \in [M]$ (using $\beta$ and $\bar N$ fresh samples)

\State Compute $\ubar{p}_{\sV_m}, \bar{p}_{\sV_m}, \ubar{p}_{\sR_m}$ using \eqref{eq:adjustment-given-radius} for $m \in [M]$

\State \Return $\smoothed{\sH}_{N,\alpha,\bsV}(x)$ and \eqref{eq:main-guarantee-clustered-smoothing} certificate.
    
\end{algorithmic}
\end{algorithm}

\footnotetext{
The quality of the smoothed predictor $\smoothed{\sH}_{N, \alpha, \bsV}(x)$ depends critically on the choice of partition $\bsV$. In Algorithm~\ref{alg:sets-construction}, we construct $\bsV$ locally around $x$ to ensure good performance in its neighborhood. However, to preserve the validity of Theorem~\ref{thm:robustness-certification-clustered-smoothing}, the same partition must be used for all points $x'$ within a $\delta$-neighborhood of $x$.
}

\review{The time complexity of Algorithm~\ref{alg:prediction_certification} is generally dominated by the repeated inference cost of the base classifier. It is thus governed by the number of samples required, which is determined by (i) the data patterns and clustering algorithm, (ii) the confidence bound (binomial inference, see Proposition~\ref{prop:high-confidence-set-construction}), and (iii) the probability bound in Theorem~\ref{thm:robustness-certification-clustered-smoothing}.}

\section{Experiments}\label{section:experiments}

We evaluated our framework on two benchmarks: stochastic trajectory prediction and multi-modal RL quadrotor control. For both experiments, we use the minimum bounding hyperrectangle for the envelope of each cluster (Appendix~\ref{section:experiment_details} provides further details on the experiment setup)\footnote{\review{The code is publicly available at \url{https://github.com/EduardoFMDCosta/ClusteredRandomizedSmoothing}.}}. All experiments were run on an Intel Core i7-1365U CPU with 16GB of RAM.

\textbf{Trajectory Prediction. }
We consider the stochastic \emph{TrajFlow} prediction model~\cite{meszaros_trajflow_2024}, a normalizing flow-based model (details in Appendix~\ref{subsection:trajectory_prediction_details}), and apply it to the \emph{L-GAP} driving simulator dataset~\cite{zgonnikov_should_2024}. In this dataset, human drivers face a decision whether to turn left in front of an oncoming vehicle (\emph{go}) or wait for it to pass before turning (\emph{yield}). The recorded trajectories are converted into control inputs (acceleration and curvature) using the unicycle model~ \cite{schumann2025realistic}. The resulting stochastic predictor is then defined as $h_\randvar{w}: \realNum^{18} \to \review{\realNum^{40}}$, where the input space $\realNum^{18}$ encodes the two control actions that underlie the $9$ transitions between the $10$ past recorded positions of the target agent. The output space \review{$\realNum^{40}$} corresponds to the position of the target vehicle \review{over $20$ time-steps}, here corresponding to a horizon of $\SI{2}{s}$.

\review{We start with a safety analysis comparing our approach to $\alpha$-smoothing and RS-Reg \cite{rekavandi2024rs} (which is equivalent to smoothing without any trimming, i.e. $\alpha=0$ in \eqref{eq:def-alpha-smoothing}) using the L-GAP dataset. In particular, we evaluate the predicted probability that the vehicle will adopt the behavior \textit{stay}\footnote{\review{The behaviors \emph{stay} and \emph{go} are defined in the Appendix~\ref{section:experiment_details} (\emph{Behavioral clustering}).}} and we make the ego vehicle \textit{go} if it predicts a \textit{stay} probability for the other higher than a decision threshold of $95\%$. We then use the actual trajectory of the predicted vehicle to assess the risk of this decision: if both vehicles decided to \textit{go}, we classify it as a \textit{risky decision}. We then report the risk rate for $200$ inputs in our database. We show that RS-Reg has a risk rate of $7.5\%$, $\alpha$-smoothing of $16.5\%$, while ours only $2.5\%$. The main mechanism behind this result is that while $\alpha$-smoothing and RS-Reg are often highly sure of the oncoming vehicle's behavior (due to the mode collapse for both RS-Reg and $\alpha$-smoothing), our clustered smoother allows for behavioral uncertainty that makes it more cautious about adopting the \textit{go} decision, as illustrated in Figure~\ref{fig:trajectory-prediction-safety-analysis}.}
\begin{figure*}
    \centering
    \includegraphics[width=1.0\textwidth]{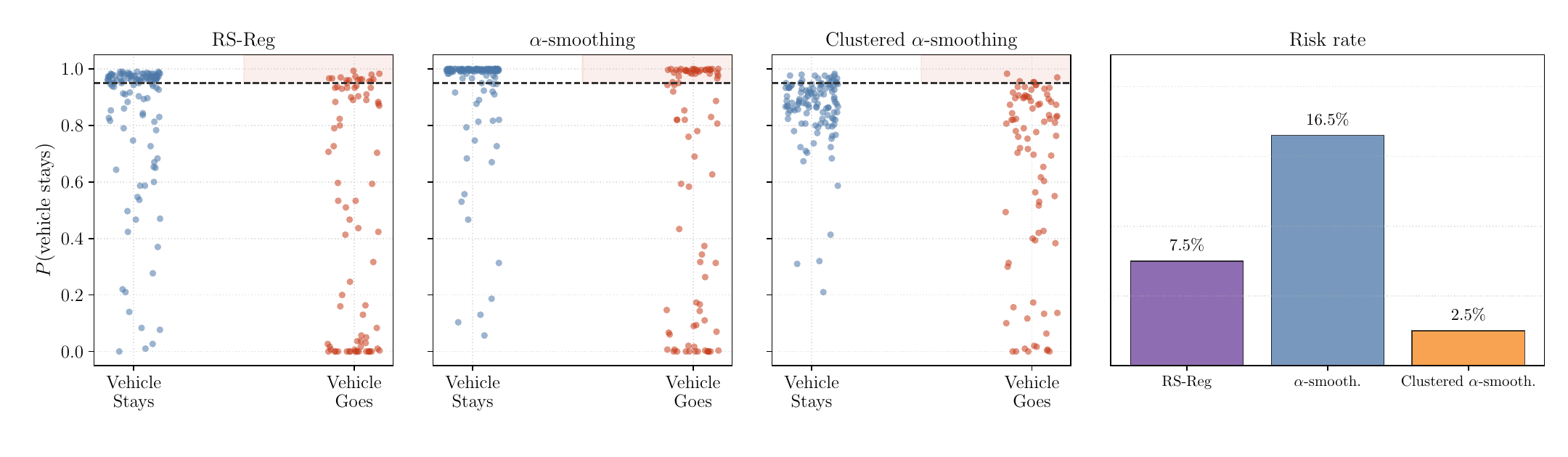}
    \vspace{-0.5cm}
    \caption{ \review{Safety analysis for $200$ randomly selected inputs $x_i$ from the L-GAP dataset.} }
    \label{fig:trajectory-prediction-safety-analysis}
    \vspace{-0.5cm}
\end{figure*}

\review{Second, for a quantitative comparison of the discrepancy between the output distributions from each approach, we select a random subset of $200$ input trajectories $x_i$ from L-GAP and compute a sample approximation of the $2$-Wasserstein distance $\wasserstein_2$ between the noisy predictor $h_\randvar{w}(x_i+\bm{\varepsilon})$ and $\alpha$-smoothing, RS-Reg, and ours $-$ note that we are thus comparing distributions embedded in a $40$-dimensional space. We found that our clustered $\alpha$-smoothing reduces the mean $2$-Wasserstein distance to $h_\randvar{w}(x+\bm{\varepsilon})$ in this subset of data by $27\%$ compared to $\alpha$-smoothing ($3.80$ vs. $5.18$; Table \ref{tab:wasserstein-data}). This confirms the intuition that the distributional flexibility of our method generally yields predictions that are closer to the original predictor, thereby better capturing its uncertainty (as also illustrated in Figure~\ref{fig:trajectory-prediction-samples-experiment} in Appendix~\ref{section:additional-experimental-results}).}

\begin{figure}[htbp]
\centering

\newsavebox{\figbox}
\sbox{\figbox}{\includegraphics[width=0.48\textwidth, height=3.7cm]{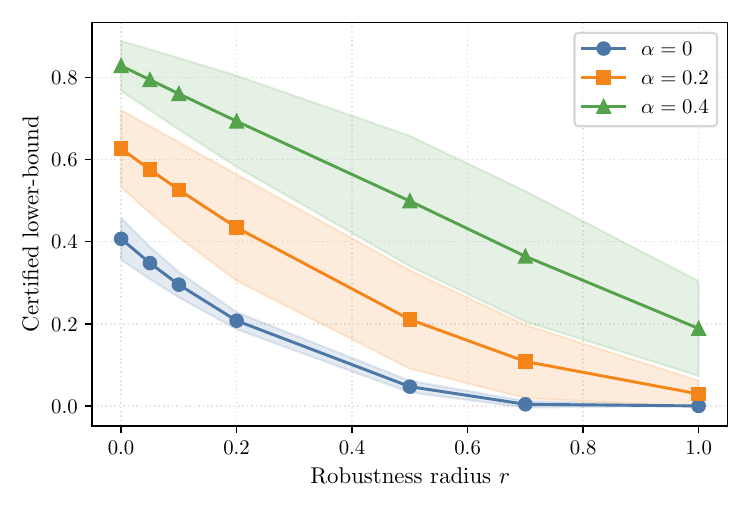}}

\begin{minipage}[c][\ht\figbox+\dp\figbox]{0.48\textwidth}
    \centering
    \vspace{0.0cm}
    \renewcommand{\arraystretch}{1.35}
    \begin{tabular*}{\linewidth}{@{\extracolsep{\fill}}lcc@{}}
    \toprule
                       & \multicolumn{2}{c}{Wasserstein $\wasserstein_2$} \\ \cmidrule(l){2-3}
    Method             & Trajectory ($\realNum^{40}$) & Last step ($\realNum^2$) \\ \midrule
    Ours               & $\mathbf{3.80 \pm 1.59}$      & $\mathbf{1.49 \pm 0.63}$   \\
    RS-Reg             & $4.70 \pm 2.09$      & $1.84 \pm 0.84$   \\
    $\alpha$-smoothing & $5.18 \pm 2.19$      & $2.04 \pm 0.84$   \\ \bottomrule
    \end{tabular*}
    \renewcommand{\arraystretch}{1}
    \captionof{table}{ \review{Comparison between smoothers and $h_\randvar{w}(x_i+\bm{\varepsilon})$ in terms of $\wasserstein_2$ for $200$ inputs $x_i$. } }
    \label{tab:wasserstein-data}
\end{minipage}
\hfill
\begin{minipage}[c][\ht\figbox+\dp\figbox]{0.48\textwidth}
    \centering
    \usebox{\figbox}
    \vspace{-0.8cm}
    \caption{ \review{Robustness lower-bounds (with one std. dev. shaded band) for $20$ distinct inputs $x_i$.} }
    \label{fig:robustness-bounds-traj-pred}
\end{minipage}
\end{figure}

\review{Finally, we perform a parameter analysis of the robustness lower-bounds in Theorem~\ref{thm:robustness-certification-clustered-smoothing} for various trimming parameters $\alpha \in \{ 0.0, 0.2, 0.4 \}$ and robustness radii $r\geq 0$, where the sets $\sR$ are constructed according to Algorithm~\ref{alg:sets-construction} enveloping a fraction $p=0.85$, for different inputs in our database. In Figure~\ref{fig:robustness-bounds-traj-pred}, we observe that certified lower bounds are higher when the robustness radius is smaller, as this allows smaller perturbations from the reference input, and when the trimming parameter $\alpha$ is higher (so that outliers are filtered out, so we can be more certain about where the smoother probability mass is concentrated). Despite the solid certification suggested by Figure~\ref{fig:robustness-bounds-traj-pred}, decreasing the conservativeness of the lower bounds in Theorem~\ref{thm:robustness-certification-clustered-smoothing} constitutes an interesting research direction.}

\textbf{Multi-modal quadrotor control. } 
Quadrotors are safety-critical systems where robustness directly affects deployability. We consider a navigation task in which the quadrotor must reach a goal while avoiding numerous obstacles (see Figure~\ref{fig:quadrotor_env_and_rollout}), under stochastic dynamics with full state observability \cite{badings2023robust}. Details of \review{the} benchmark can be found in Appendix~\ref{subsec:quadrotor_details}. The task is inherently multi-modal: obstacles between the quadrotor and its goal admit passage on either side, and wind can shift probability mass between such modes.
\begin{figure}
    \centering
    \begin{subfigure}{0.3\textwidth}
        \centering
        \includegraphics[width=\textwidth]{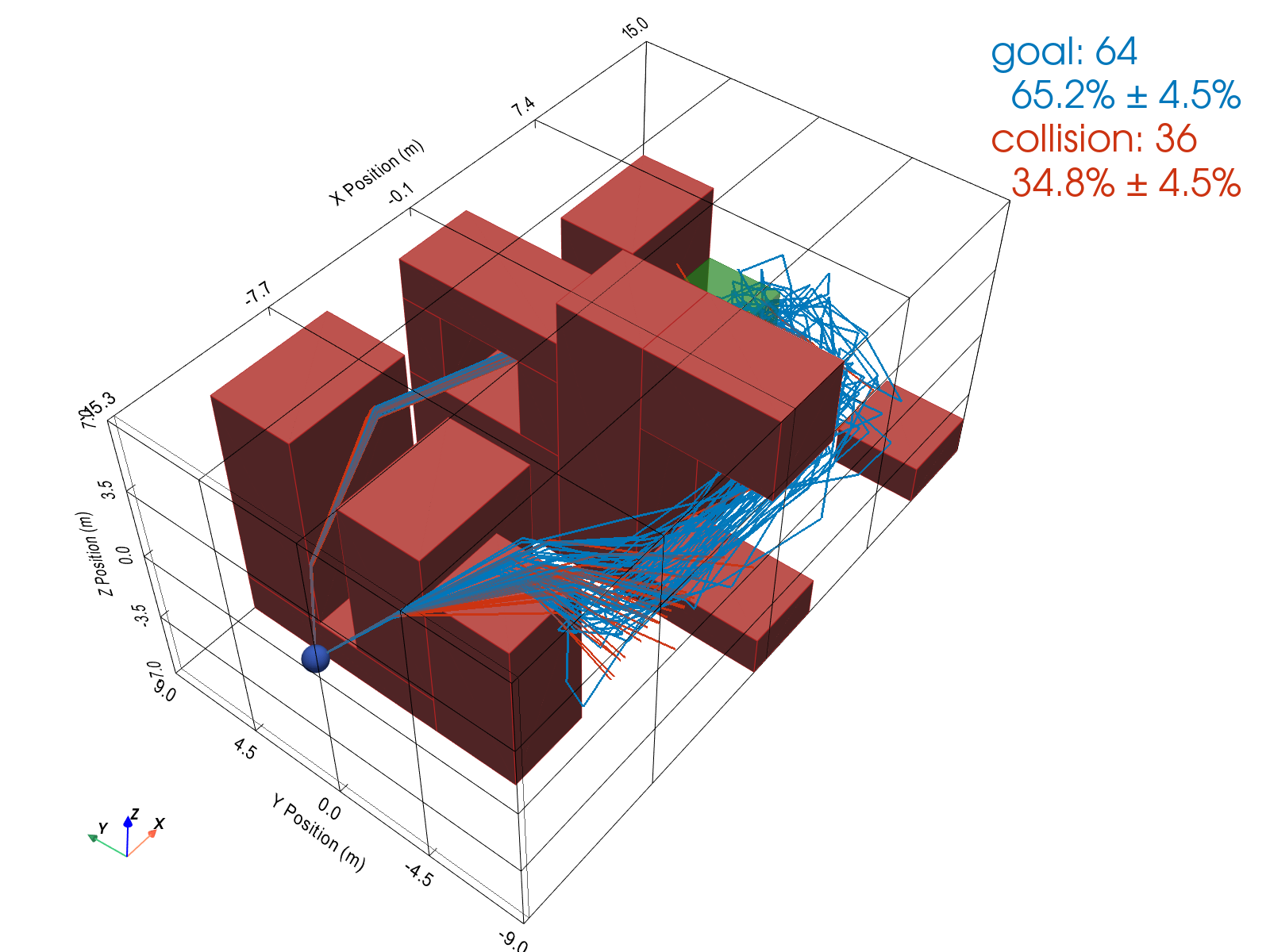}\\[0.3em]
        \includegraphics[width=0.85\textwidth]{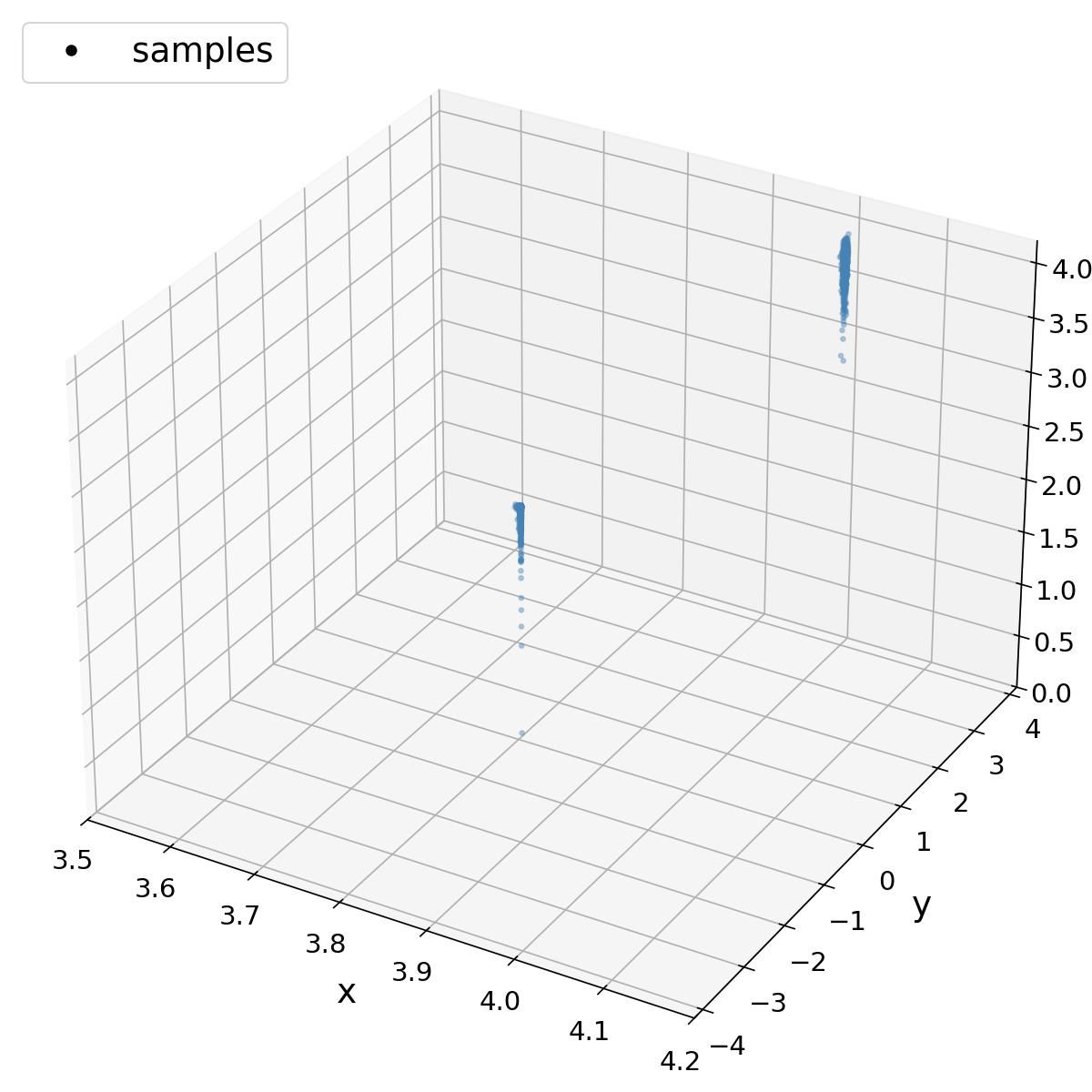}
        \caption{nominal policy}
    \end{subfigure}%
    \hfill
    \begin{subfigure}{0.3\textwidth}
        \centering
        \includegraphics[width=\textwidth]{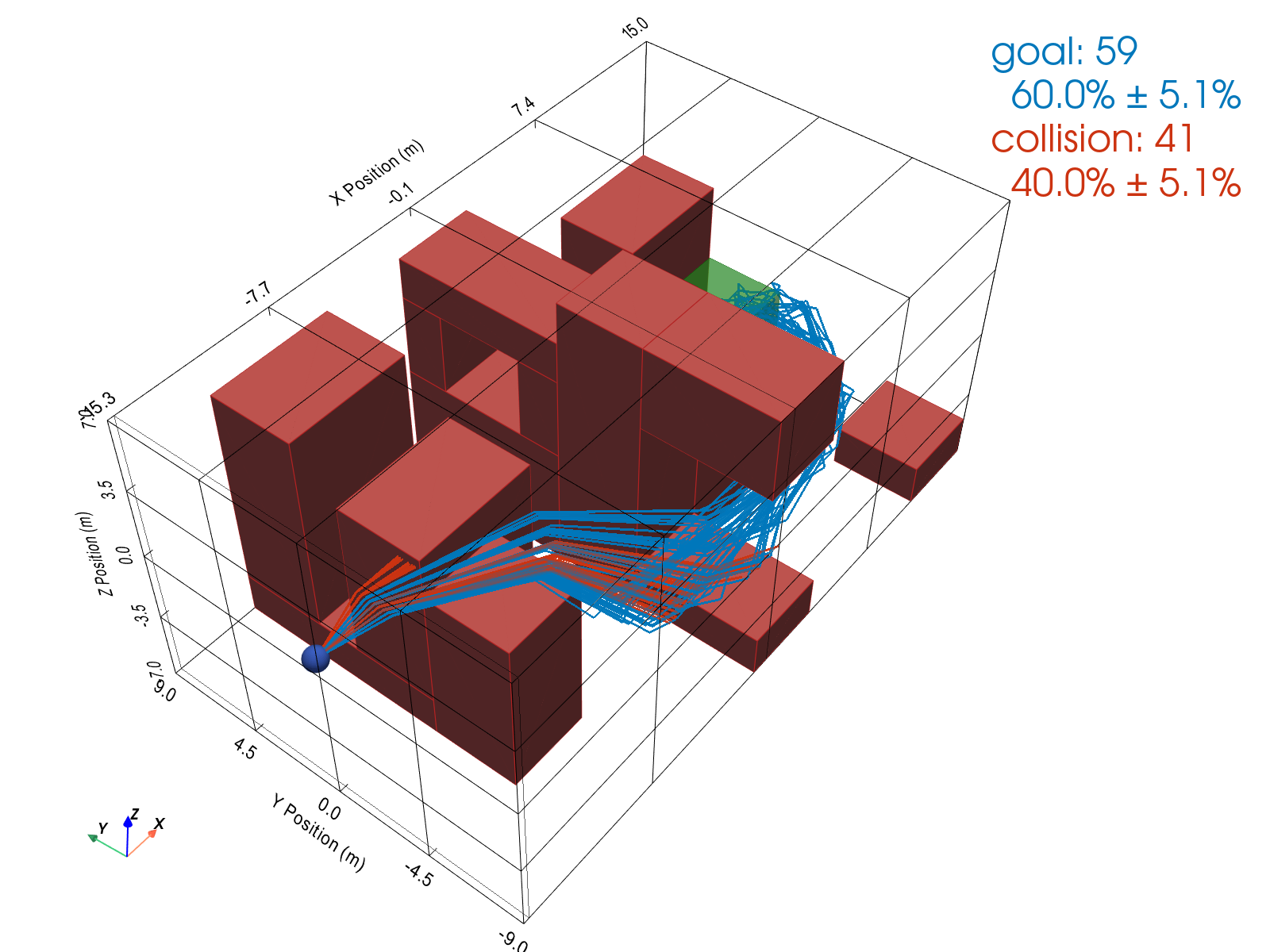}\\[0.3em]
        \includegraphics[width=0.85\textwidth]{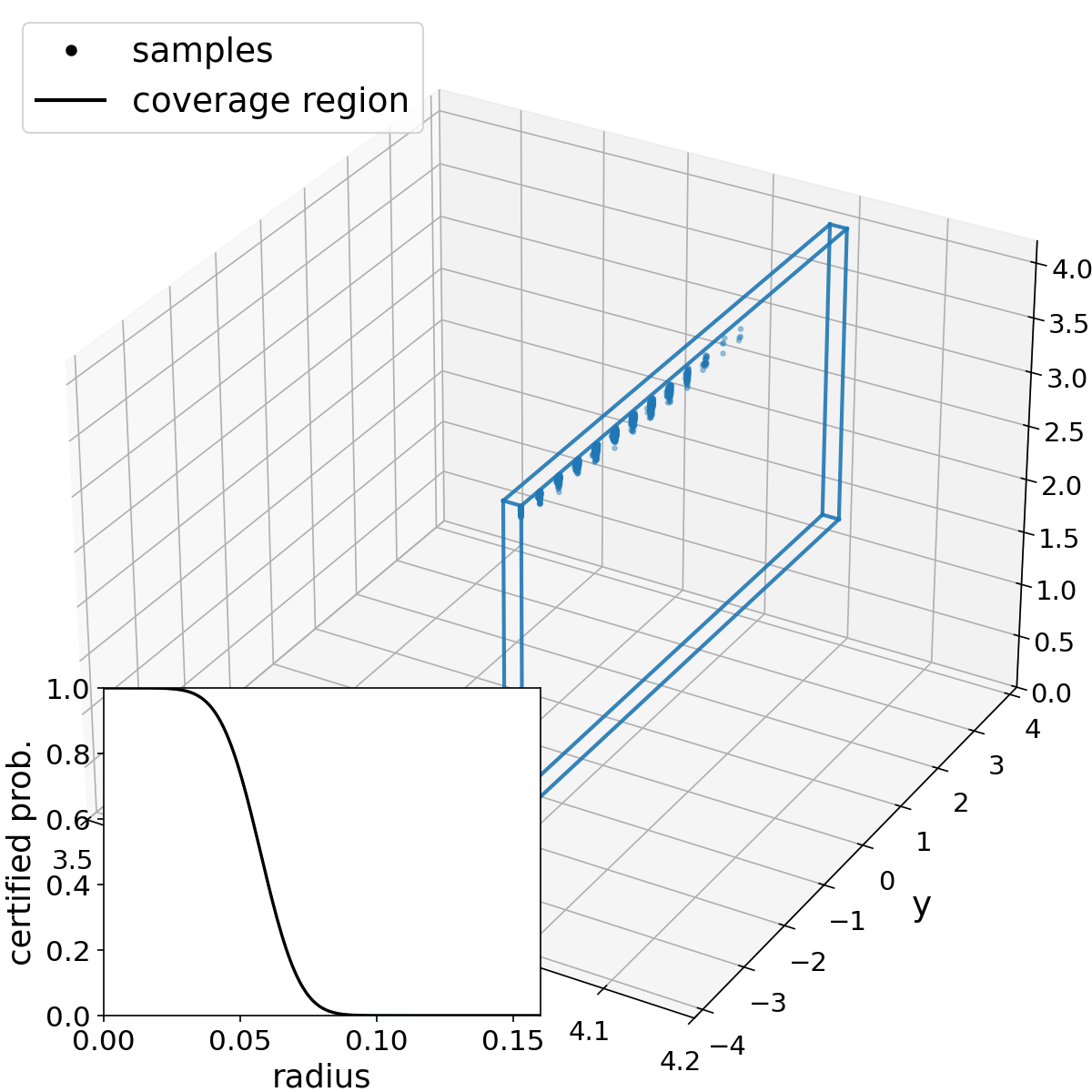}
        \caption{$\alpha$-smoothing \cite{rekavandi2024certified}}
    \end{subfigure}%
    \hfill
    \begin{subfigure}{0.3\textwidth}
        \centering
        \includegraphics[width=\textwidth]{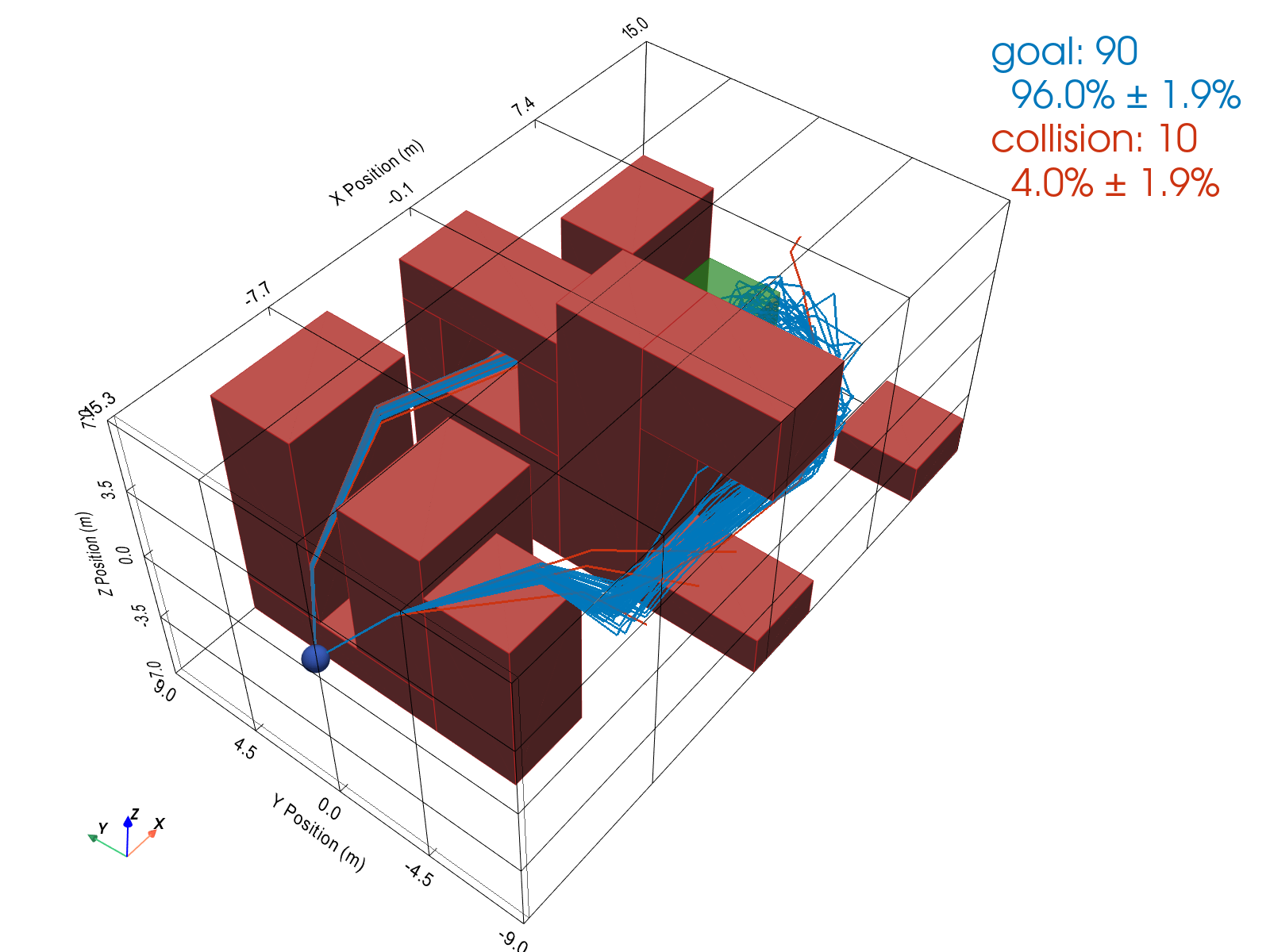}\\[0.3em]
        \includegraphics[width=0.85\textwidth]{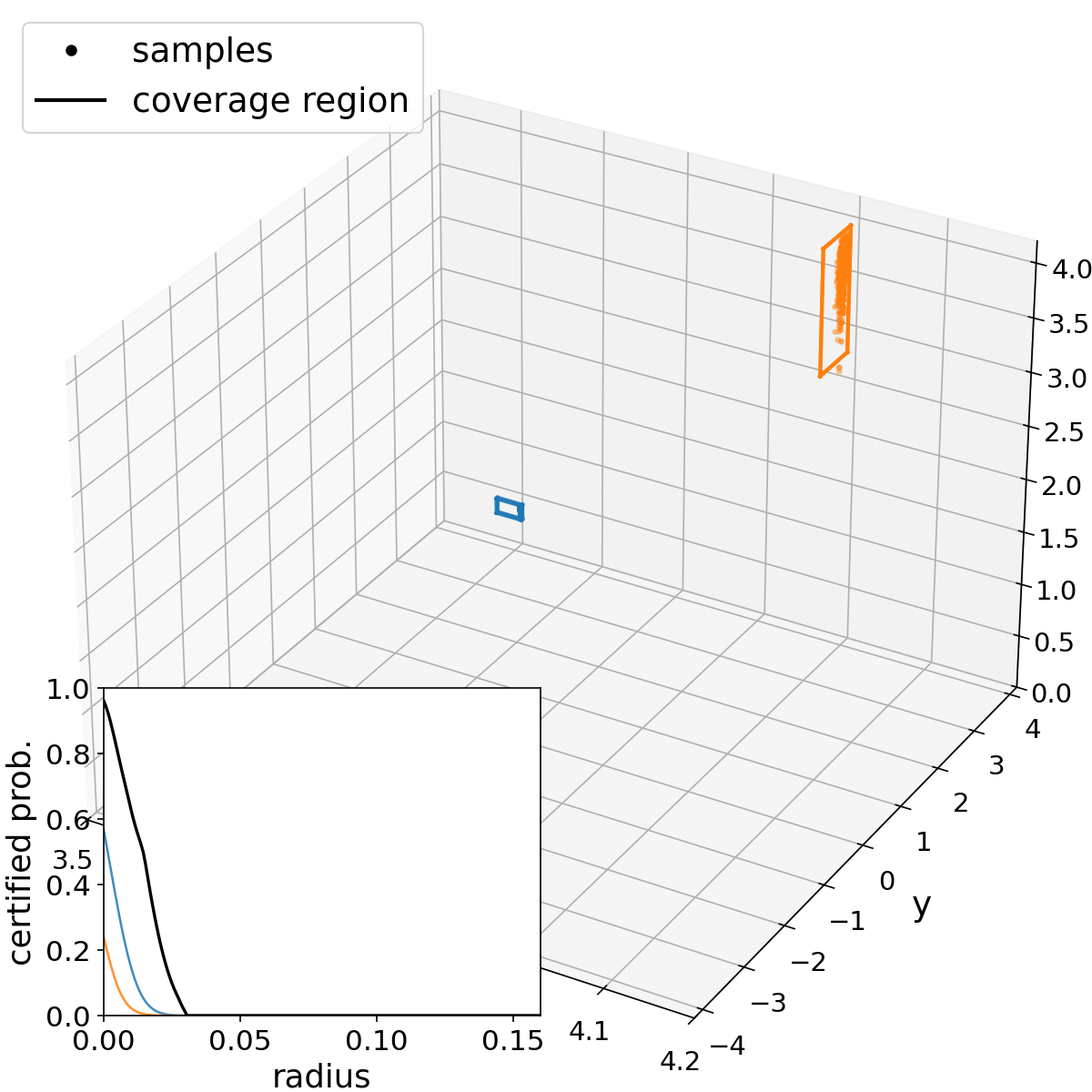}
        \caption{clustered $\alpha$-smoothing (ours)}
    \end{subfigure}
    \caption{
    Top row: rollouts of the quadrotor from the starting state (blue marker) where the controller predicts a bi-modal action distribution, corresponding to navigating through two ``holes'' in the obstacles towards the goal region (green). \review{The top number represents the number of trajectories (out of 100) reaching the goal in the depicted simulation, while the numbers below represent the average and standard deviation of the goal rate over 100 simulations of 100 trajectories; similar for the collision rate.} Bottom row: action distributions at the initial state $x = (-14.8, 0, 0.7, 0, 0, 0)$. The smoothed predictors are constructed using 30 samples. Insets: the certified probability of robustness of the smoothed policies with respect to their coverage regions (boxes) as a function of perturbation radius. The quantization in (b) arises from cross-mode smoothing with a limited number of samples (30 samples, $\alpha$-trimmed to 18) (see Appendix~\ref{section:implicit_quantization} for the mechanism behind quantization).}
    \label{fig:quadrotor_env_and_rollout}
    \vspace{-0.5cm}
\end{figure} 
The controller is an RL policy that outputs a Gaussian mixture over actions with predicted component means and mixture weights, and fixed variances \cite{ren2021probabilistic}. The prediction of a mixture enables navigation around obstacles through either path and rejection of wind-induced disturbances. This is precisely the setting where randomized and $\alpha$-smoothing fail: perturbations are averaged such that different modes are collapsed to their (weighted) mean, which often is to navigate directly into an obstacle. Clustered smoothing recovers meaningful guarantees by treating each cluster individually.

Visual inspection of policies (Figure~\ref{fig:quadrotor_env_and_rollout}) shows that the nominal policy is bi-modal, corresponding to navigating through one of two paths through the obstacles to the goal. However, the $\alpha$-smoothing destroys the multi-modality and ultimately predicts the (trimmed) average, which corresponds to navigating directly at the obstacles. Instead, our framework preserves the modes and only tightens the clusters. In the simulation shown in Figure~\ref{fig:quadrotor_env_and_rollout}, the nominal policy reaches the goal in 64 of the 100 trajectories sampled and crashes in 36 trajectories. The $\alpha$-smoothed policy crashes in 41 trajectories, while our clustered $\alpha$-smoothed policy only crashes in 10 trajectories. 
The certified probability as a function of perturbation radius in Figure~\ref{fig:quadrotor_env_and_rollout} suggests that our method is penalized more for larger input radii $r$, by the correction of $\ubar{p}_{\sV_m}, \bar{p}_{\sV_m}$ in \eqref{eq:adjustment-given-radius}. \review{One might question if the smoothing requires excessive acceleration or jerk, hence being unusable in practice. To this end, we report in Table~\ref{tab:quadrotor_kinematic_stats} the mean and standard deviation of the acceleration and jerk for both the base policy and smoothed predictors. The results show that the mean acceleration and jerk remain within one standard deviation of the base policy, which is also comparable to $\alpha$-smoothing. Furthermore, the variation also remains comparable. We report in Appendix~\ref{subsection:additional_quadrotor_results} additional experimental results for the quadrotor including the Wasserstein distance to the base policy across different states and the rollout performance under adversarial attack.} 

\begin{table}[]
    \centering
    \caption{Statistics of the kinematics for the experiment in Figure~\ref{fig:quadrotor_env_and_rollout}.}
    \label{tab:quadrotor_kinematic_stats}
    \begin{tabular}{r|ccc}
        \toprule
        & Base & $\alpha$-smoothing \cite{rekavandi2024certified} & Clustered smoothing \\\midrule
        Acceleration & $5.609 \pm 0.924$ & $5.535 \pm 0.935$ & $5.469 \pm 0.876$ \\
        Jerk & $9.175 \pm 2.600$ & $9.375 \pm 1.805$ & $9.165 \pm 2.485$ \\\bottomrule
    \end{tabular}
\end{table}

\review{Note that the guarantees of the smoothed predictors are with respect to the controller output at each time step, while the benchmark operates in a receding horizon fashion. Nevertheless, the empirical results show significant benefit in applying our method compared to both the base policy and regular $\alpha$-smoothing. Moreover, our approach is agnostic to the underlying predictor and is thus applicable to controllers such as MPC and motion planners such as RRT, both of which are important methods in robotic applications.}

\section{Limitations}\label{section:limitations}

\review{Randomized smoothing relies on repeated inference. In certain contexts (e.g. large neural networks), this may be prohibitive, in particular when the base predictor has a non-negligible computational cost. That our method is agnostic to the clustering algorithm is both a strength and a weakness: the quality of the result depends strongly on the choice, and thus, one will need to assess the clustering to choose a suitable algorithm. Theory-wise, we identify four primary limitations. First, defining $\sR$ as a union of disjoint \reviewtheory{axis-aligned hyperrectangles} is a key component to enabling the guarantees, but this is restrictive, e.g., if the clusters cannot be cleanly separated \reviewtheory{by such boxes, as in} the famous half-moon toy example. The second limitation -- that certification is with respect to the smoothed predictor -- is key, as one cannot apply our method to robustify the prediction if parity with the base regressor is crucial, such as physics and dynamical systems \emph{simulations}. Third, in \eqref{eq:adjustment-given-radius}, if one wants to reduce the noise but maintain the probability bound, it is necessary to fix $r/\sigma$, and thus, $r$ must be proportional to $\sigma$. Fourth, the lower bounds from Theorem~\ref{thm:robustness-certification-clustered-smoothing} can become more conservative with the inclusion of new clusters. Consequently, either a lower certified radius or a larger noise level may be needed to guarantee robustness in highly multi-modal settings.}

\section{Conclusion}\label{sec:conclusion}

In this work, we addressed a key limitation of randomized smoothing for regression: the collapse of multi-modal predictions under averaging. We proposed \emph{clustered $\alpha$-smoothing}, which preserves multi-modality by applying $\alpha$-trimmed smoothing within clusters and combining the results as a mixture. We established probabilistic robustness guarantees and demonstrated significant empirical gains, reducing the Wasserstein distance by \review{$27\%$} in trajectory prediction and the collision rate by $90\%$ in quadrotor control. 
Several directions remain for future work. In particular, our current approach assumes a fixed partition \review{$\bsV$} of the output space \review{$\sY$}. However, recent work in randomized smoothing \cite{sukenik2021intriguing, alfarra2022data} suggests that input-dependent smoothing distributions can improve performance. Extending our framework to allow input-dependent smoothing and partitions $\bsV(x)$ is therefore a promising and challenging direction.

\begin{ack}
This research is partially supported by the NWO (grant OCENW.M.22.056) and by the EPSRC grant EP/Y028872/1, Mathematical Foundations of Intelligence: An “Erlangen Programme” for AI.
\end{ack}

{\footnotesize
\bibliographystyle{plain}
\bibliography{main}
}

\clearpage
\appendix

\section{Technical Proofs}\label{section:proofs}

\subsection{Proof of Theorem \ref{thm:robustness-certification-clustered-smoothing}}
As Theorem \ref{thm:robustness-certification-clustered-smoothing} is built from a union-bound argument over $\sR_m$, we start by introducing the following Lemma~\ref{lem:robustness-per-region-certification-clustered-smoothing}, which bounds the probability $\Prob_{(\randvar{w}, \bm{\varepsilon}_1, \ldots, \bm{\varepsilon}_N)} \left(\smoothed{\sH}_{N, \alpha, \bsV}(x + \delta) \in \sR_m \right)$. For notational convenience, we write $\ubar{\bm{\varepsilon}}$ to mean $\bm{\varepsilon}_1, \ldots, \bm{\varepsilon}_N$ for the remainder of this appendix.
\begin{lemma}\label{lem:robustness-per-region-certification-clustered-smoothing}
    Let $\bsV = \left\{\sV_1, \dots, \sV_M\right\}$ be a partition of the space \review{$\sY$}. Assume that $\check{p}_{\sV_m} \leq \Prob_{(\randvar{w}, \bm{\varepsilon})}\left(h_\randvar{w}(x + \bm{\varepsilon}) \in \sV_m \right) \leq \hat{p}_{\sV_m}$ and $\Prob_{(\randvar{w}, \bm{\varepsilon})}\left(h_\randvar{w}(x + \bm{\varepsilon}) \in \sR_m \right) \geq \check{p}_{\sR_m}$ for \reviewtheory{an axis-aligned hyperrectangle $\sR_m = \prod_{k=1}^{q} [a_m^{(k)}, b_m^{(k)}] \subset \sV_m$}. For $r \geq 0$, define
    \begin{equation}\label{eq:adjustment-given-radius-per-region}
        \ubar{p}_{\sV_m} = \Phi\left(\Phi^{-1}(\check{p}_{\sV_m}) - \frac{r}{\sigma}\right), \; \bar{p}_{\sV_m} = \Phi\left(\Phi^{-1}(\hat{p}_{\sV_m}) + \frac{r}{\sigma}\right), \; \text{ and } \;
        \ubar{p}_{\sR_m} = \Phi\left(\Phi^{-1}(\check{p}_{\sR_m}) - \frac{r}{\sigma}\right).
    \end{equation}
    Then, for any $\delta \in \realNum^d$ such that $\norm{\delta}_2 \leq r$, \reviewtheory{letting $p^{\delta}_{\sV_m} = \Prob_{(\randvar{w}, \bm{\varepsilon})}\left(h_\randvar{w}(x + \delta + \bm{\varepsilon}) \in \sV_m \right)$,} it holds that
    \begin{align}
        &\Prob_{(\randvar{w}, \ubar{\bm{\varepsilon}})} \left(\smoothed{\sH}_{N, \alpha, \bsV}(x + \delta) \in \sR_m \right) \nonumber \\
        &\qquad \geq \reviewtheory{\sum_{s=1}^{N} \frac{s}{N} \sum_{j=s-\floor{\alpha s}}^{s} f_{s,j}\left(\frac{\ubar{p}_{\sR_m}}{p^{\delta}_{\sV_m}} \right) f_{N,s}\left(p^{\delta}_{\sV_m}\right)}.
    \end{align}
\end{lemma}
\begin{proof}
    First, we observe that using the Neyman-Pearson's Lemma \cite{cohen2019certified}, for any $\delta \in \realNum^d$ such that $\norm{\delta}_2 \leq r$,
    \begin{equation}\label{eq:proof-lower-bound-partition}
        \Prob_{(\randvar{w}, \bm{\varepsilon})}\left(h_\randvar{w}(x + \delta + \bm{\varepsilon}) \in \sV_m\right) \geq \Phi\left(\Phi^{-1}(\check{p}_{\sV_m}) - \frac{\norm{\delta}_2}{\sigma}\right) \geq \Phi\left( \Phi^{-1}(\check{p}_{\sV_m}) - \frac{r}{\sigma}\right) = \ubar{p}_{\sV_m},
    \end{equation}
    \begin{equation}\label{eq:proof-upper-bound-partition}
        \Prob_{(\randvar{w}, \bm{\varepsilon})}\left(h_\randvar{w}(x + \delta + \bm{\varepsilon}) \in \sV_m\right) \leq \Phi\left(\Phi^{-1}(\hat{p}_{\sV_m}) + \frac{\norm{\delta}_2}{\sigma}\right) \leq \Phi\left(\Phi^{-1}(\hat{p}_{\sV_m}) + \frac{r}{\sigma}\right) = \bar{p}_{\sV_m},
    \end{equation}
    and
    \begin{equation}\label{eq:proof-lower-bound-region}
        \Prob_{(\randvar{w}, \bm{\varepsilon})}\left(h_\randvar{w}(x + \delta + \bm{\varepsilon}) \in \sR_m\right) \geq \Phi\left(\Phi^{-1}(\check{p}_{\sR_m}) - \frac{\norm{\delta}_2}{\sigma}\right) \geq \Phi\left(\Phi^{-1}(\check{p}_{\sR_m}) - \frac{r}{\sigma}\right) = \ubar{p}_{\sR_m},
    \end{equation}
    due to the monotonicity of the function $\Phi$. Further, because by construction $\sR_m \subset \sV_m$, \reviewtheory{denoting by $p^{\delta}_{\sR_m}$ the probability in \eqref{eq:proof-lower-bound-region},} we have that
    \begin{equation}\label{eq:proof-lower-bound-conditional}
        \Prob_{(\randvar{w}, \bm{\varepsilon})}\left(h_\randvar{w}(x + \delta + \bm{\varepsilon}) \in \sR_m \mid h_\randvar{w}(x + \delta + \varepsilon) \in \sV_m \right) \reviewtheory{= \frac{p^{\delta}_{\sR_m}}{p^{\delta}_{\sV_m}} \geq \frac{\ubar{p}_{\sR_m}}{p^{\delta}_{\sV_m}}}.
    \end{equation}
    Having these probability bounds for $h_\randvar{w}(x + \delta + \bm{\varepsilon})$, we can now proceed to study the probability of the event $\smoothed{\sH}_{N,\alpha,\bsV}(x + \delta) \in \sR_m$. For any $x' \in \sX$, by applying the Law of Total Probability twice \reviewtheory{and dropping a nonnegative term}, it holds that
    \begin{align}
        &\Prob_{(\randvar{w}, \ubar{\bm{\varepsilon}})} \left(\smoothed{\sH}_{N, \alpha, \bsV}(x') \in \sR_m\right)= \nonumber \\
        &\Prob_{(\randvar{w}, \ubar{\bm{\varepsilon}})}\left(\smoothed{\sH}_{N, \alpha, \bsV}(x') \in \sR_m \mid \smoothed{\sH}_{N, \alpha, \sV_m}(x') \in \sR_m\right)\Prob_{(\randvar{w}, \ubar{\bm{\varepsilon}})} \left(\smoothed{\sH}_{N, \alpha, \sV_m}(x') \in \sR_m\right) + \nonumber \\
        &\qquad \underbrace{\Prob_{(\randvar{w}, \ubar{\bm{\varepsilon}})}\left(\smoothed{\sH}_{N, \alpha, \bsV}(x') \in \sR_m \mid \smoothed{\sH}_{N, \alpha, \sV_m}(x') \notin \sR_m \right)}_{\reviewtheory{\geq 0}}\Prob_{(\randvar{w}, \ubar{\bm{\varepsilon}})} \left(\smoothed{\sH}_{N, \alpha, \sV_m}(x') \notin \sR_m\right) \mathrel{\reviewtheory{\geq}} \nonumber \\
        &\sum_{s=0}^N \Prob_{(\randvar{w}, \ubar{\bm{\varepsilon}})}\left(\smoothed{\sH}_{N, \alpha, \bsV}(x') \in \sR_m \mid \smoothed{\sH}_{N,\alpha, \sV_m}(x') \in \sR_m, |I_m|=s \right)\cdot\nonumber\\
        &\qquad\qquad \Prob_{(\randvar{w}, \ubar{\bm{\varepsilon}})}\left( \smoothed{\sH}_{N,\alpha, \sV_m}(x') \in \sR_m \mid |I_m|=s \right) \Prob_{(\randvar{w}, \ubar{\bm{\varepsilon}})}(|I_m|=s).\label{eq:proof-main-equation-component-wise}
    \end{align}
    
    In particular, let $x' = x + \delta$. We will examine each of these terms separately. Recall that, by construction, the random set of \review{indices} $I_m$ is given by $I_m = \left\{ i\in [N] : h_\randvar{w}(x + \delta + \varepsilon_i) \in \sV_m \right\}$, so that \reviewtheory{$|I_m| \sim \text{Binomial}\left( N, p^{\delta}_{\sV_m} \right)$}. Therefore,
    \begin{equation}\label{eq:proof-control-component-1}
        \Prob_{(\randvar{w}, \ubar{\bm{\varepsilon}})}(|I_m|=s) \reviewtheory{= f_{N,s}\left( p^{\delta}_{\sV_m} \right)}.
    \end{equation}

    We then move on to the event $\smoothed{\sH}_{N, \alpha, \bsV}(x + \delta) \in \sR_m$ given $\smoothed{\sH}_{N,\alpha, \sV_m}(x + \delta) \in \sR_m, |I_m|=s$. Trivially, conditioned on $\smoothed{\sH}_{N,\alpha, \sV_m}(x + \delta) \in \sR_m$, the event $\smoothed{\sH}_{N, \alpha, \bsV}(x + \delta) \in \sR_m$ takes place if the component $\smoothed{\sH}_{N,\alpha, \sV_m}$ is sampled from the mixture, which by construction is equal to $\frac{|I_m|}{N}$. Thus, \reviewtheory{as the event can also occur when another component is selected,}
    \begin{equation}\label{eq:proof-control-component-2}
        \Prob_{(\randvar{w}, \ubar{\bm{\varepsilon}})} \left( \smoothed{\sH}_{N, \alpha, \bsV}(x + \delta) \in \sR_m \mid \smoothed{\sH}_{N,\alpha, \sV_m}(x + \delta) \in \sR_m, |I_m|=s \right) \mathrel{\reviewtheory{\geq}} \frac{s}{N}.
    \end{equation}
    
    Finally, to conclude, we need to lower-bound the probability of the event $\smoothed{\sH}_{N,\alpha, \sV_m}(x + \delta) \in \sR_m$ given $|I_m|=s$. We observe that conditioned on the knowledge $|I_m|=s$, the random variable $\smoothed{\sH}_{N,\alpha, \sV_m}(x + \delta)$ conditioned on $|I_m|=s$ behaves as the $\alpha$-smoother introduced in \cite{rekavandi2024certified} with number of samples $s$. Therefore, we follow a similar proof. Let $Z_m = \sum_{i=1}^s \indicator_{h_\randvar{w}(x + \delta + \bm{\varepsilon}_i) \in \sR_m \mid h_\randvar{w}(x + \delta + \bm{\varepsilon}_i) \in \sV_m}$, thus 
    \begin{equation}\label{eq:proof-def-zm-variable}
        Z_m \sim \text{Binomial}\left( s, \Prob_{(\randvar{w}, \bm{\varepsilon})}\left( h_\randvar{w}(x + \delta + \bm{\varepsilon}) \in \sR_m \mid h_\randvar{w}(x + \delta + \bm{\varepsilon}) \in \sV_m \right) \right).
    \end{equation}
    From the Law of Total Probability, it holds that
    \begin{align}\label{eq:eq-aux-proof-total-prob}
        &\Prob_{(\randvar{w}, \ubar{\bm{\varepsilon}})} \left( \smoothed{\sH}_{N,\alpha, \sV_m}(x + \delta) \in \sR_m \mid |I_m|=s \right) = \nonumber \\ 
        &\qquad\quad \sum_{z=0}^s \Prob_{(\randvar{w}, \ubar{\bm{\varepsilon}})} \left( \smoothed{\sH}_{N,\alpha, \sV_m}(x + \delta) \in \sR_m \mid |I_m|=s, Z_m=z \right)\Prob(Z_m=z),
    \end{align}
    Now, note that \reviewtheory{when $z \geq s - \floor{\alpha s}$, at most $\floor{\alpha s}$ of the $s$ samples lie outside $\sR_m$. Fix a coordinate $k \in [q]$, and let $n_k^{+}$ (resp.\ $n_k^{-}$) be the number of samples whose $k$-th coordinate is larger than $b_m^{(k)}$ (resp.\ smaller than $a_m^{(k)}$). Each of these samples lies outside $\sR_m$, so $n_k^{+} + n_k^{-} \leq \floor{\alpha s}$. Hence, the $n_k^{+}$ samples exceeding $b_m^{(k)}$ are among the $\floor{\alpha s}$ largest in the $k$-th coordinate, and the $n_k^{-}$ samples below $a_m^{(k)}$ are among the $\floor{\alpha s}$ smallest, so they are all discarded by the trimming in \eqref{eq:intra-cluster-alpha-trimming-def}. Thus, all values retained in the $k$-th coordinate lie in $[a_m^{(k)}, b_m^{(k)}]$, and so does their average. Since this holds for every $k \in [q]$, we obtain $\smoothed{\sH}_{N,\alpha, \sV_m}(x + \delta) \in \sR_m$ conditioned on the given events. Because the trimming is element-wise, different coordinates may average values from different samples (some of which may lie outside $\sR_m$), so the argument is carried out per coordinate, which is where we use that $\sR_m$ is a product of intervals.} On the other hand, \reviewtheory{when $z < s - \floor{\alpha s}$}, the average in \eqref{eq:intra-cluster-alpha-trimming-def} contains points out of $\sR_m$. In the absence of more information about how far those points are, we have no guarantees that the average over them will belong to $\sR_m$ (as only one point suffices to displace the average arbitrarily far from $\sR_m$ \cite{rekavandi2024rs}). Therefore, conditioned on these given events, we may conservatively consider that $\smoothed{\sH}_{N, \alpha, \sV_m}(x + \delta) \notin \sR_m$. Thus, from \eqref{eq:eq-aux-proof-total-prob},
    \begin{align}
        &\Prob_{(\randvar{w}, \ubar{\bm{\varepsilon}})} \left( \smoothed{\sH}_{N,\alpha, \sV_m}(x + \delta) \in \sR_m \mid |I_m|=s \right) = \nonumber \\
        &\qquad \sum_{z=0}^{s-\floor{\alpha s}-1} \underbrace{\Prob_{(\randvar{w}, \ubar{\bm{\varepsilon}})} \left( \smoothed{\sH}_{N,\alpha, \sV_m}(x + \delta) \in \sR_m \mid |I_m|=s, Z_m=z \right)}_{\geq 0}\Prob_{Z_m}(Z_m = z) + \nonumber \\
        & \qquad\quad \sum_{z=s-\floor{\alpha s}}^{s} \underbrace{ \Prob_{(\randvar{w}, \ubar{\bm{\varepsilon}})} \left( \smoothed{\sH}_{N,\alpha, \sV_m}(x + \delta) \in \sR_m \mid |I_m|=s, Z_m=z \right) }_{= 1}\Prob_{Z_m}(Z_m = z) \nonumber \\
        &\geq \sum_{z=s-\floor{\alpha s}}^{s} \Prob_{Z_m}(Z_m = z) \nonumber \\
        &\reviewtheory{\geq \sum_{j=s-\floor{\alpha s}}^{s} f_{s,j} \left( \frac{\ubar{p}_{\sR_m}}{p^{\delta}_{\sV_m}} \right)}\label{eq:proof-control-component-3}
    \end{align}
    from the fact that $Z_m$ is distributed as the binomial in \eqref{eq:proof-def-zm-variable} and its success probability is lower bounded by \eqref{eq:proof-lower-bound-conditional}\reviewtheory{, and that the upper tail of a binomial distribution is nondecreasing in its success probability}. The proof is concluded by replacing \eqref{eq:proof-control-component-1}, \eqref{eq:proof-control-component-2}, and \eqref{eq:proof-control-component-3} in \eqref{eq:proof-main-equation-component-wise}.
\end{proof}

\begin{corollary}\label{cor:per-region-closed-form}
    \reviewtheory{In the setting of Lemma~\ref{lem:robustness-per-region-certification-clustered-smoothing}, for any $\delta \in \realNum^d$ such that $\norm{\delta}_2 \leq r$,}
    \begin{align*}
        &\Prob_{(\randvar{w}, \ubar{\bm{\varepsilon}})} \left(\smoothed{\sH}_{N, \alpha, \bsV}(x + \delta) \in \sR_m \right) \\
        &\qquad \geq \reviewtheory{\sum_{s=1}^{N} \frac{s}{N} \sum_{j=s-\floor{\alpha s}}^{s} f_{s,j}\left(\frac{\ubar{p}_{\sR_m}}{\bar{p}_{\sV_m}} \right) \min\left\{ f_{N,s}\left(\ubar{p}_{\sV_m}\right), \; f_{N,s}\left(\bar{p}_{\sV_m}\right) \right\}}.
    \end{align*}
\end{corollary}
\begin{proof}
    \reviewtheory{By Lemma~\ref{lem:robustness-per-region-certification-clustered-smoothing}, it suffices to lower bound its right-hand side, where $p^{\delta}_{\sV_m} \in [\ubar{p}_{\sV_m}, \bar{p}_{\sV_m}]$ by \eqref{eq:proof-lower-bound-partition} and \eqref{eq:proof-upper-bound-partition}. First, $\ubar{p}_{\sR_m}/p^{\delta}_{\sV_m} \geq \ubar{p}_{\sR_m}/\bar{p}_{\sV_m}$, and the upper tail $\sum_{j=s-\floor{\alpha s}}^{s} f_{s,j}(\rho)$ of a binomial distribution is nondecreasing in its success probability $\rho$. Second, $f_{N,s}(p)=\binom{N}{s}p^s(1-p)^{N-s}$ has a positive derivative for $p < \frac{s}{N}$ and a negative one for $p > \frac{s}{N}$, so its minimum over $[\ubar{p}_{\sV_m}, \bar{p}_{\sV_m}]$ is attained at one of the endpoints, i.e., $f_{N,s}(p^{\delta}_{\sV_m}) \geq \min\{ f_{N,s}(\ubar{p}_{\sV_m}), f_{N,s}(\bar{p}_{\sV_m}) \}$.}
\end{proof}

We are now ready to prove Theorem~\ref{thm:robustness-certification-clustered-smoothing}.
\begin{proof}
    \reviewtheory{By assumption, we have $\sR_l \cap \sR_{l'} = \emptyset$ for $l, l' \in \sL$ where $l \neq l'$. Fix $\delta \in \realNum^d$ such that $\norm{\delta}_2 \leq r$, and let $p^{\delta}_{\sV_m} = \Prob_{(\randvar{w}, \bm{\varepsilon})}\left(h_\randvar{w}(x + \delta + \bm{\varepsilon}) \in \sV_m \right)$ for $m \in [M]$. Applying Lemma~\ref{lem:robustness-per-region-certification-clustered-smoothing} to each $l \in \sL$ and adding over the disjoint sets $\sR_l$, we obtain}
    \begin{align}\label{eq:uncorrected-clustered-smoothing}
        &\Prob_{(\randvar{w}, \ubar{\bm{\varepsilon}})} \left(\smoothed{\sH}_{N, \alpha, \bsV}(x + \delta) \in \tilde{\sR} \right) \nonumber \\
        &\qquad = \sum_{l \in \sL} \Prob_{(\randvar{w}, \ubar{\bm{\varepsilon}})} \left(\smoothed{\sH}_{N, \alpha, \bsV}(x + \delta) \in \sR_l \right) \nonumber \\
        &\qquad \geq \reviewtheory{\sum_{l \in \sL} \sum_{s=1}^{N} \frac{s}{N} \sum_{j=s-\floor{\alpha s}}^{s} f_{s,j} \left( \frac{\ubar{p}_{\sR_l}}{p^{\delta}_{\sV_l}} \right) f_{N,s} \left(p^{\delta}_{\sV_l}\right)}.
    \end{align}
    \reviewtheory{The right-hand side depends on $\delta$ only through $(p^{\delta}_{\sV_1}, \dots, p^{\delta}_{\sV_M})$, which is feasible for the optimization problem below: $p^{\delta}_{\sV_m} \in [\ubar{p}_{\sV_m}, \bar{p}_{\sV_m}]$ by \eqref{eq:proof-lower-bound-partition} and \eqref{eq:proof-upper-bound-partition}, and $\sum_{m = 1}^M p^{\delta}_{\sV_m} = 1$ because $\bsV$ is a partition of $\sY$. Hence, the right-hand side is lower bounded by the infimum over all such vectors, and we arrive at}
    \begin{align}\label{eq:corrected-clustered-smoothing}
        &\Prob_{(\randvar{w}, \ubar{\bm{\varepsilon}})} \left(\smoothed{\sH}_{N, \alpha, \bsV}(x + \delta) \in \tilde{\sR} \right) \nonumber \\
        &\qquad \geq \inf_{\substack{p_{\sV_m} \in [\ubar{p}_{\sV_m}, \bar{p}_{\sV_m}] \\ \sum_{m=1}^M p_{\sV_m}=1}} \sum_{l \in \sL} \sum_{s=1}^{N} \frac{s}{N} \sum_{j=s-\floor{\alpha s}}^{s} f_{s,j} \left( \frac{\ubar{p}_{\sR_l}}{p_{\sV_l}} \right) f_{N,s} \left(p_{\sV_l}\right),
    \end{align}
    which concludes the proof.
\end{proof}

\subsection{Proof of Proposition~\ref{prop:prob_bound_with_anchor_points}}
Before proving Proposition~\ref{prop:prob_bound_with_anchor_points}, we first prove \reviewtheory{two auxiliary lemmas}.
\begin{lemma}\label{lemma:lower-bound-as-lp}
    Let $g:[\ell,u]\to\realNum$ be a Lipschitz continuous function with Lipschitz constant $L_g$. Let $\left\{ \evp{1},\dots, \evp{K} \right\}$ be a set of $K \in \natNum_{>0}$ anchor points such that $\ell=\evp{1}<\evp{2}<\dots<\evp{K}=u$. Let $\operatorname{conv}_K(g)(p)$ be defined as
    \begin{equation}\label{eq:conv-envelope-approximation-definition}
        \operatorname{conv}_K(g)(p) = \min_{\bm{\lambda} \in \Delta_K} \left\{ \sum_{k=1}^K \evlambda{k}g(\evp{k}) \text{ s.t. } \sum_{k=1}^K \evlambda{k}\evp{k} = p \right\}.
    \end{equation}
    Then, \reviewtheory{for every $p \in [\ell, u]$,}
    \begin{equation}
        \operatorname{conv}_K(g)(p) \leq g(p) + L_g h_K,
    \end{equation}
    where $h_K = \max_{p'\in[\ell, u]} \min_{k\in[K]} |p' - \evp{k}|$.
\end{lemma}
\begin{proof}
    \reviewtheory{Fix $p \in [\ell, u]$ and let $\evp{k} \leq p \leq \evp{k+1}$ be consecutive anchors, which exist since $\evp{1} = \ell$ and $\evp{K} = u$. Let $a = p - \evp{k}$, $b = \evp{k+1} - p$, $\Delta = a + b$ and $\lambda = b/\Delta \in [0,1]$, so that $1-\lambda = a/\Delta$ and $p = \lambda \evp{k} + (1-\lambda) \evp{k+1}$. The vector $\bm{\lambda}=[0,\dots,0,\lambda, 1-\lambda, 0,\dots,0]$ (non-zero at positions $k$ and $k+1$) belongs to $\Delta_K$ and satisfies the constraint in \eqref{eq:conv-envelope-approximation-definition}. Hence, using in the third line that, by the Lipschitz property, $|g(\evp{k}) - g(p)| \leq L_g a$ and $|g(\evp{k+1}) - g(p)| \leq L_g b$, and substituting $\lambda = b/\Delta$ and $1-\lambda = a/\Delta$ in the fourth line,
    \begin{align*}
        \operatorname{conv}_K(g)(p) &\leq \lambda g(\evp{k}) + (1-\lambda) g(\evp{k+1}) \\
        &= g(p) + \lambda \big( g(\evp{k}) - g(p) \big) + (1-\lambda) \big( g(\evp{k+1}) - g(p) \big) \\
        &\leq g(p) + L_g \big( \lambda a + (1-\lambda) b \big) \\
        &= g(p) + L_g \Big( \frac{b}{\Delta}\, a + \frac{a}{\Delta}\, b \Big) \\
        &= g(p) + L_g \, \frac{2ab}{\Delta} \leq g(p) + L_g \, \frac{\Delta}{2} \leq g(p) + L_g h_K,
    \end{align*}
    where the second-to-last step uses $4ab \leq (a+b)^2 = \Delta^2$, and the last one holds because the midpoint of $[\evp{k}, \evp{k+1}]$ is at distance $\Delta/2$ from its closest anchors, so $\Delta/2 \leq h_K$.}
\end{proof}

\begin{lemma}\label{lemma:computing-lipschitz-constant}
Let $g:(0,1]\to\realNum$ be the function defined by
\begin{equation*}
    g(p)= \reviewtheory{\sum_{s=1}^{N} \frac{s}{N} f_{N,s}(p) \sum_{j=s-\floor{\alpha s}}^{s} f_{s,j}\left(\frac{\tilde p}{p}\right),}
\end{equation*}
where $\alpha \in [0,\frac{1}{2})$, $N \in \natNum_{>0}$, and \reviewtheory{$0 \leq \tilde p < \ubar{p} \leq \bar{p} < 1$}. Then $g$ is Lipschitz continuous on $[\ubar{p},\bar{p}]$ with Lipschitz constant
\begin{equation}
    L_g \leq \frac{N}{1-\bar{p}} + \frac{N}{\ubar{p}-\tilde p}.
\end{equation}
\end{lemma}
\begin{proof}
\reviewtheory{Since $g$ is differentiable on $[\ubar{p},\bar{p}]$, $L_g=\sup_{p\in[\ubar{p},\bar{p}]}|g'(p)|$. For $p \in [\ubar{p},\bar{p}]$ we have $0 \leq \tilde p/p < 1$ and, using $1 - \tilde p/p = (p - \tilde p)/p$,
\begin{equation*}
    f_{N,s}(p)\, f_{s,j}\Big(\frac{\tilde p}{p}\Big) = T_{s,j}(p) := C_{s,j}\, \tilde p^{\,j} (p - \tilde p)^{s-j} (1-p)^{N-s}, \qquad C_{s,j} = \binom{N}{s}\binom{s}{j},
\end{equation*}
so that $g = \sum_{s=1}^N \frac{s}{N} \sum_{j=s-\floor{\alpha s}}^{s} T_{s,j}$. As $p - \tilde p > 0$ and $1-p > 0$,
\begin{equation*}
    T_{s,j}'(p) = T_{s,j}(p) \left[ \frac{s-j}{p - \tilde p} - \frac{N-s}{1-p} \right].
\end{equation*}
Both terms in the bracket are nonnegative, so its absolute value is at most their maximum, which is at most $\Lambda := \max\big\{ \frac{N}{\ubar{p}-\tilde p}, \frac{N}{1-\bar{p}} \big\}$. Hence, by the triangle inequality, $|g'(p)| \leq \Lambda\, g(p)$. Since $g(p) \leq \sum_{s=1}^N f_{N,s}(p) \leq 1$, we conclude that
\begin{equation*}
    L_g \leq \Lambda \leq \frac{N}{1-\bar{p}}+\frac{N}{\ubar{p}-\tilde p}.
\end{equation*}}
\end{proof}

We are now ready to prove Proposition~\ref{prop:prob_bound_with_anchor_points}.
\begin{proof}
    \reviewtheory{Let $\mathcal{F} = \big\{ \bm{p} \in \realNum^M : p_{\sV_m} \in [\ubar{p}_{\sV_m}, \bar{p}_{\sV_m}] \text{ for all } m \in [M], \; \sum_{m=1}^M p_{\sV_m} = 1 \big\}$. By definition of $g_l$, the left-hand side of \eqref{eq:bound-union-lp-formulation} is $\inf_{\bm{p} \in \mathcal{F}} \sum_{l \in \sL} g_l(p_{\sV_l})$. For each $l \in \sL$, Lemma~\ref{lemma:computing-lipschitz-constant}, applicable by the assumptions of the proposition with $\tilde p = \ubar{p}_{\sR_l}$, shows that $g_l$ is Lipschitz continuous on $[\ubar{p}_{\sV_l}, \bar{p}_{\sV_l}]$ with constant
    \begin{equation*}
        L_{g_l} \leq \frac{N}{1 - \bar{p}_{\sV_l}} + \frac{N}{\ubar{p}_{\sV_l} - \ubar{p}_{\sR_l}}.
    \end{equation*}
    Since the anchors include both endpoints, Lemma~\ref{lemma:lower-bound-as-lp} yields $\operatorname{conv}_K(g_l)(p_{\sV_l}) \leq g_l(p_{\sV_l}) + L_{g_l} h_{l, K}$ for every $p_{\sV_l} \in [\ubar{p}_{\sV_l}, \bar{p}_{\sV_l}]$. Summing over $l \in \sL$ and taking the infimum over $\mathcal{F}$,
    \begin{equation*}
        \inf_{\bm{p} \in \mathcal{F}} \sum_{l \in \sL} \operatorname{conv}_K(g_l)(p_{\sV_l}) \leq \inf_{\bm{p} \in \mathcal{F}} \sum_{l \in \sL} g_l(p_{\sV_l}) + \sum_{l \in \sL} L_{g_l} h_{l, K}.
    \end{equation*}
    Finally, the minimizations over the $\bm{\lambda}_l$ in the definition of $\operatorname{conv}_K(g_l)$ are independent across $l \in \sL$, so the left-hand side is the infimum of the objective $\sum_{l \in \sL} \sum_{k=1}^K \lambda_l^{(k)} g_l(p_{\sV_l}^{(k)})$ over the constraints in \eqref{eq:union_objective}. Rearranging and using the bound on $L_{g_l}$ above concludes the proof.}
\end{proof}

\subsection{Proof of Proposition \ref{prop:high-confidence-set-construction}}

Before proving Proposition \ref{prop:high-confidence-set-construction}, we introduce the well-known Clopper-Pearson Lemma.

\begin{lemma}[Clopper-Pearson lower bound]\label{lemma:clopper-pearson}
    Let $\{ \randvar{z}_1,\dots,\randvar{z}_M \}$ be i.i.d. samples from $\Prob$, and let $h : \sX \to \sY$ and $\sR \subseteq \sY$ be given. Define $Z = \sum_{i=1}^M \indicator_{h(\randvar{z}_i) \in \sR}$. Given a confidence level $\beta \in [0,1]$, if $\check{p} \in [0,1]$ is such that $\beta = \sum_{j=Z}^M f_{M,j}(\check{p})$, then 
    \begin{equation}
        \Prob_{(\randvar{z}_1,\dots,\randvar{z}_M)} \left( \Prob_{\randvar{z}}\left( h(\randvar{z}) \in \sR \right) \geq \check{p} \right) \geq 1 - \beta
    \end{equation}
\end{lemma}

We are now ready to prove Proposition \ref{prop:high-confidence-set-construction}.
\begin{proof}
    From the Fréchet bound and Lemma \ref{lemma:clopper-pearson},
    \begin{align*}
        &\Prob_{ \{ (\randvar{w}_i,\varepsilon_i) \}_{i=1}^N } \left( \cap_{m=1}^M \left\{ \check{p}_{\sV_m} \leq \Prob_{(\randvar{w},\varepsilon)}\left( h_\randvar{w}(x+\varepsilon) \in \sV_m \right) \leq \hat{p}_{\sV_m}, \check{p}_{\sR_m} \leq \Prob_{(\randvar{w},\varepsilon)}\left( h_\randvar{w}(x+\varepsilon) \in \sR_m \right) \right\} \right) \\
        &\qquad \geq \sum_{i=1}^{3M} \left( 1-\frac{\beta}{3M} \right) - (3M-1) = 1 - \beta.
    \end{align*}
\end{proof} 

\section{Experiment Details}\label{section:experiment_details}
The clustered $\alpha$-smoothing certificate is parameterized by three groups. The clustering group depends on the strategy. For the purposes of the experiments in Section~\ref{section:experiments}, we use DBSCAN. 
The parameters are listed in Table~\ref{tab:parameter_overview}.

\begin{table}[ht!]
\centering
\small
\caption{Clustered $\alpha$-smoothing parameters.}\label{tab:parameter_overview}
\begin{tabular}{llp{0.42\linewidth}}
\toprule
Parameter & Symbol & Impact \\
\midrule
\multicolumn{3}{l}{\emph{Clustering (DBSCAN strategy)}} \\
Maximum number of clusters & $M_{\max}$ & If the output of DBSCAN has more than $M_{\max}$ clusters, then we merge them in order of fewest to most samples. \\
Coverage  & $p$        & See Line 9 of Algorithm~\ref{alg:sets-construction} \\
Neighborhood radius & $\varepsilon_\text{db}$ & DBSCAN neighbourhood radius in output space. \\
Min. samples & --      & DBSCAN core-point threshold. \\
\midrule
\multicolumn{3}{l}{\emph{Smoothing core}} \\
Std.dev. of Gaussian input-perturbation     & $\sigma$     & Larger std.dev. allows larger perturbation radius $r$, at the cost of larger deviation from $h_\randvar{w}(x)$. See \eqref{eq:adjustment-given-radius}. \\
Number of samples for smoothing & $N$    & More samples tightens the modes of the smoothed distribution. \\
Trimming fraction & $\alpha$      & The certificate discards the most extreme $2\alpha$ tails of each per-cluster output coordinate, $0 \le \alpha < 1/2$. \\
\midrule
\multicolumn{3}{l}{\emph{Confidence}} \\
Number of samples for confidence & $\bar{N}$    & Number of samples used for Proposition~\ref{prop:high-confidence-set-construction}. More samples leads to better confidence, at an increased computational cost. \\
Confidence parameter & $\beta$ & The joint confidence for all estimated probability bounds is $1 - \beta$. \\
\bottomrule
\end{tabular}
\end{table}

\subsection{Trajectory Prediction}\label{subsection:trajectory_prediction_details}

\paragraph{\emph{TrajFlow} architecture and training procedure.} 
\emph{TrajFlow}~\cite{meszaros_trajflow_2024} is built with a combination of a normalizing flow network $F$ with a recurrent autoencoder, consisting of encoder $E_{\text{RNN}}$ and decoder $D_{\text{RNN}}$ (see Figure~\ref{fig:trajflow_architecture}). Furthermore, the model is completed with the networks $\phi_{\text{RNN}}$ (a gated \review{recurrent} unit encoding the predicted \review{agent's} past trajectory), $\psi_{\text{GNN}}$ (a graph neural network encoding the other agents), and $\phi_{\text{CNN}}$ (a convolutional network encoding the environment, which is given as a bird's-eye image). 
This model is trained in a two-stage process, where first the recurrent autoencoder is trained to encode the future trajectories of all involved agents. Once done, the normalizing flow is trained on the encoded future sampled as well as the past and environment information. A detailed description of the model components and the training loss can be found in the original work ~\cite{meszaros_trajflow_2024}.

\begin{figure}
    \centering
    \includegraphics[width=0.9\linewidth]{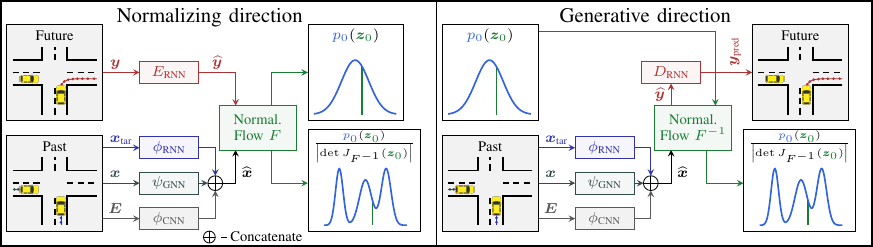}
    \caption{An overview of \emph{TrajFlow}'s architecture, taken from Figure~2 in ~\cite{meszaros_trajflow_2024}. During training, future trajectories $\bm{y}$ are encoded with the encoder $E_{\text{RNN}}$ and transformed to the abstracted features $\widehat{\bm{y}}$. Using the normalizing direction, they are then transformed into a sample $\bm{z}_0 = F(\widehat{\bm{y}})$, which follows a standard normal distribution $p_0$.
    For inference we then use the generative direction, in which a sample $\bm{z}_0 \sim p_0$ is inversely transformed by the Normalizing Flow to generate the abstracted future trajectories $\widehat{\bm{y}} = F^{-1}(\bm{z}_0)$ that are decoded with $D_{\text{RNN}}$ into the actual trajectories $\bm{y}_\text{pred}$.
    The likelihood of the encoded trajectory $\widehat{\bm{y}}$ is obtained with $p_0 (\bm{z}_0) \vert \det J_{F^{-1}}(\bm{z}_0)\vert^{-1}$. 
    The encoding $\phi_{\text{CNN}}$ of map $E$, and the encoding $\psi_{\text{GNN}}$ of social interactions are optional blocks, which can provide richer context information.}
    \label{fig:trajflow_architecture}
\end{figure}

For model training, we employ the \emph{STEP} framework~\cite{schumann2025step} (MIT License) to train \emph{TrajFlow} for our chosen scenario. In particular, the model is trained on a dataset combining \emph{NuScenes}~\cite{caesar_nuscenes_2020} (CC BY-NC-SA) and \emph{L-GAP}~\cite{zgonnikov_should_2024} (CC-BY). Following previous works~\cite{schumann2025realistic,schumann2026evaluating}, we do not use the standard input for \emph{NuScenes} (i.e., 4 past time steps at a frequency of $\SI{2}{Hz}$), but instead use \emph{STEP} to extract past trajectories with 10 time steps at $\SI{10}{Hz}$. This choice is mostly motivated by the fact that this results in a smoother trajectory, for which we can extract the underlying control states more easily. Importantly, while these datasets provide trajectories with states including positions, velocities and accelerations, for \emph{TrajFlow}, we only consider the recorded positions. The code is available at 
\url{https://github.com/DAI-Lab-HEReALD/General-Framework/blob/main/Framework/simulations_randomized_smoothing.py}.


\paragraph{Behavioral clustering.} For our experiments, \review{we run Algorithm \ref{alg:sets-construction} with a behavior-based two-cluster heuristic in endpoint space, using prior knowledge of two tactical modes (\emph{stay} and \emph{go}). For each sample, let $\Delta = y_{\mathrm{end}} - y_{\mathrm{init}}$. A sample is tagged as \emph{stay} when $\lVert \Delta \rVert_2 \leq r_{\mathrm{stay}}$, and as \emph{go} when $\Delta^{(2)} \leq -\delta_{\downarrow}$ and $\Delta^{(1)} \le -\delta_{\leftarrow}$, excluding already-tagged \emph{stay} samples. Endpoint prototypes are then built from these tagged sets (with fallback rules if one set is empty), and all samples are assigned to the nearest prototype in $2$D endpoint space. Finally, for each cluster $j \in \{1,2\}$, we construct a tight axis-aligned hyperrectangle $\sR_j$ in the full trajectory space ($40$D) that covers a fraction $p$ of the samples assigned to that cluster.}

\subsection{Quadrotor Control}\label{subsec:quadrotor_details}
The quadrotor control benchmark was first described in \cite{badings2023robust}. We use it solely to obtain a safety-critical, stochastic, multi-modal closed-loop system on which to perform clustered randomized smoothing. Relative to \cite{badings2023robust}, we augment the system with a reward structure and extended termination criteria. 

\paragraph{Environment.}
The state is $s_k = [x_k, \dot{x}_k, y_k, \dot{y}_k, z_k, \dot{z}_k] \in \realNum^6$ and the action $u_k = [u^x_k, u^y_k, u^z_k] \in [-4, 4]^3$. The dynamics are a per-axis double integrator with unit time step,
\begin{equation}
    s_{k+1} = A s_k + B u_k + C w_k, \qquad w_k \sim \mathcal{N}(0, \sigma^2 I)
\end{equation}
with velocities clipped to $[-7, 7]$ after each update. In particular, per axis the dynamics are $x_{k+1} = x_k + \tau \dot{x}_k + \tfrac{1}{2} \tau^2 u^x_k$ and $\dot{x}_{k+1} = \dot{x}_k + \tau u^x_k + w_k$ where $w_k \sim \mathcal{N}(0, \sigma^2)$. The corridor is $[-15,15] \times [-9,9] \times [-7,7]$, the goal region is $[11,15] \times [1,5] \times [-7,-3]$, and 14 axis-aligned hyperrectangular obstacles match \cite{badings2023robust}. The termination criteria are \emph{in order}:
\begin{enumerate}
    \item if $[x_k, y_k, z_k]$ is outside the corridor,
    \item if \emph{any} point on the line segment between $[x_k, y_k, z_k]$ and $[x_{k+1}, y_{k+1}, z_{k+1}]$ touches an obstacle,
    \item if $[x_k, y_k, z_k]$ is inside the goal region, or
    \item if the max episode length is reached. 
\end{enumerate}
The parameters of the environment are summarized in Table~\ref{tab:quadrotor_env_parameters}.

\begin{table}[ht!]
\centering
\small
\caption{Environment parameters. $\Delta\text{dist}$ denotes the change in distance $\|\text{pos} - \text{goal}\|$ to the goal center $\text{goal} = (13, 3, -5)$ between subsequent steps.}\label{tab:quadrotor_env_parameters}
\begin{tabular}{lll}
\toprule
Quantity & Value & Notes \\
\midrule
State / action dim & 6 / 3 & \\
Action bounds & $[-4, 4]^3$ & matches JAIR benchmark \\
Process noise $\sigma$ & $0.05$ & on velocity coords; small enough that obstacles dominate \\
Velocity clip & $[-7, 7]$ & prevents PPO blow-up under exploration \\
Max steps / episode & $64$ & matches rollout horizon \\
Reward (goal) & $+10$ & terminal \\
Reward (collision / exit) & $-5$ & terminal \\
Reward (timeout) & $-0.05 \cdot \|\text{pos} - \text{goal}\|$ & distance-shaped: far timeouts penalised more \\
Reward (alive) & $0.5 \cdot \Delta\text{dist} - 0.01$ & dense progress shaping \\
State normalization & $s / [15,7,9,7,7,7]$ & per-dim rescale for the policy net \\
\bottomrule
\end{tabular}
\end{table}

\paragraph{Policy and value networks.}
The policy is a tanh-squashed Gaussian mixture: at state $s$, sample a mode $z \sim \mathrm{Cat}(\mathrm{softmax}(g_\phi(s)))$ where $\mathrm{Cat}$ is a categorical distribution, sample an unscaled action from this mode $r \sim \mathcal{N}(\mu_\theta^z(s), \mathrm{diag}\, e^{2 \log \sigma_z})$, and rescale the action $a = \tanh(r) \cdot a_\text{scale}$. The mixture results in a multi-modal policy, which indeed requires \emph{clustered} $\alpha$-smoothing. The architecture is summarized in Table~\ref{tab:quadrotor_arch_parameters}.

\begin{table}[ht!]
\centering
\small
\caption{Network details. $A$ denotes the action dimension.}\label{tab:quadrotor_arch_parameters}
\begin{tabular}{ll}
\toprule
Component & Specification \\
\midrule
Trunk & 2 $\times$ Linear(128) + tanh, orthogonal init (gain $\sqrt{2}$) \\
Mean head & Linear($K \cdot A$), orthogonal init gain $0.01$ \\
Mixture head ($K > 1$) & Linear($K$), orthogonal init gain $0.01$ \\
Log-std & global per-component $[K, A]$, init $-1.2$, clamped to $[-5, 2]$ \\
Critic & identical 2 $\times$ Linear(128) + tanh trunk, scalar head \\
$K$ (mixture components) & $2$ \\
$a_\text{scale}$ & $4.0$ (matches action bound; tanh saturation gives the true control limit) \\
\bottomrule
\end{tabular}
\end{table}

\paragraph{Training.}
The policy is trained using standard clipped PPO with GAE and a curriculum on the start state. The curriculum \review{uses} a linear schedule where a scalar $f \in [0, 1]$ linearly anneals from $0$ to $1$ over the first $40{,}000$ episodes. At each rollout we sample a valid (collision-free, outside the goal) start position uniformly inside an axis-aligned box around the goal center with half-width $r = 3 + 25 f$, and initial velocities $\mathcal{N}(0, (0.1 + 0.4 f)^2)$. Without curriculum learning, the policy frequently terminates in collisions before it ever sees the goal. The remaining training parameters are listed in Table~\ref{tab:quadrotor_training_parameters}.

\begin{table}[ht!]
\centering
\small
\caption{PPO hyperparameters.}\label{tab:quadrotor_training_parameters}
\begin{tabular}{lll}
\toprule
Hyperparameter & Value & Notes \\
\midrule
Optimizer & Adam, lr $3 \times 10^{-4}$ & shared across actor and critic \\
Total episodes & $200{,}000$ & \\
Rollout horizon $T$ & $64$ steps & matches env timeout \\
Episodes per update & $32$ & $\approx 2$k transitions / batch \\
PPO epochs / batch & $10$ & \\
Minibatch size & $512$ & \\
Discount $\gamma$ & $0.99$ & \\
GAE $\lambda$ & $0.95$ & \\
Clip range & $0.2$ & \\
Value loss coefficient & $0.5$ & \\
Entropy coefficient & $0.02$ & mild; the mixture itself adds exploration \\
Max grad norm & $0.5$ & global clip across actor + critic \\
Advantage normalization & per-batch (mean / std) & \\
\bottomrule
\end{tabular}
\end{table}

\paragraph{Randomized smoothed policy parameters.}
Figure~\ref{fig:quadrotor_env_and_rollout} is produced with the clustered $\alpha$-smoothing parameters in Table~\ref{tab:quadrotor_smoothing_parameters} below. $\alpha$-smoothing uses the same parameters, except for no clustering.

\begin{table}[ht!]
\centering
\small
\caption{Clustered $\alpha$-smoothing parameters for the quadrotor benchmarks.}\label{tab:quadrotor_smoothing_parameters}
\begin{tabular}{lll}
\toprule
Parameter & Value & Notes \\
\midrule
$\sigma$            & $0.05$ & \\
$N$                 & $30$ & \\
$\alpha$            & $0.2$ & \\
$M_{\max}$          & $3$ & \\
Coverage $p$        & $0.9$ & \\
Clustering alg.     & DBSCAN & \\
$\varepsilon_\text{db}$ & $0.45$ & Output-space DBSCAN radius. \\
\texttt{dbscan\_min\_samples} & $50$ & Core-point threshold relative. \\
$\bar{N}$                 & $4000$ & Samples for Proposition~\ref{prop:high-confidence-set-construction}. \\
$\beta$             & $6\cdot10^{-2}$ & \\
\bottomrule
\end{tabular}
\end{table}

\section{Additional Experimental Results}\label{section:additional-experimental-results}

\subsection{Trajectory Prediction}
We examine a specific input instance from the dataset and demonstrate that the introduction of a small input perturbation $\bm{\varepsilon} \sim \sN(0, I)$ increases the behavioral variability of the target agent by raising the probability of the tactical behavior \emph{go}. While $\alpha$-smoothing tends to suppress the \emph{go} mode -- thereby suggesting that the target vehicle is likely to yield -- our clustered $\alpha$-trimming preserves the predicted multi-modality of $h_\randvar{w}(x+\bm{\varepsilon})$, \review{as illustrated in} Figure \ref{fig:trajectory-prediction-samples-experiment}, \review{where we show the last time step of the sampled trajectories}.
\begin{figure*}[ht]
    \centering
    \includegraphics[width=1.0\textwidth]{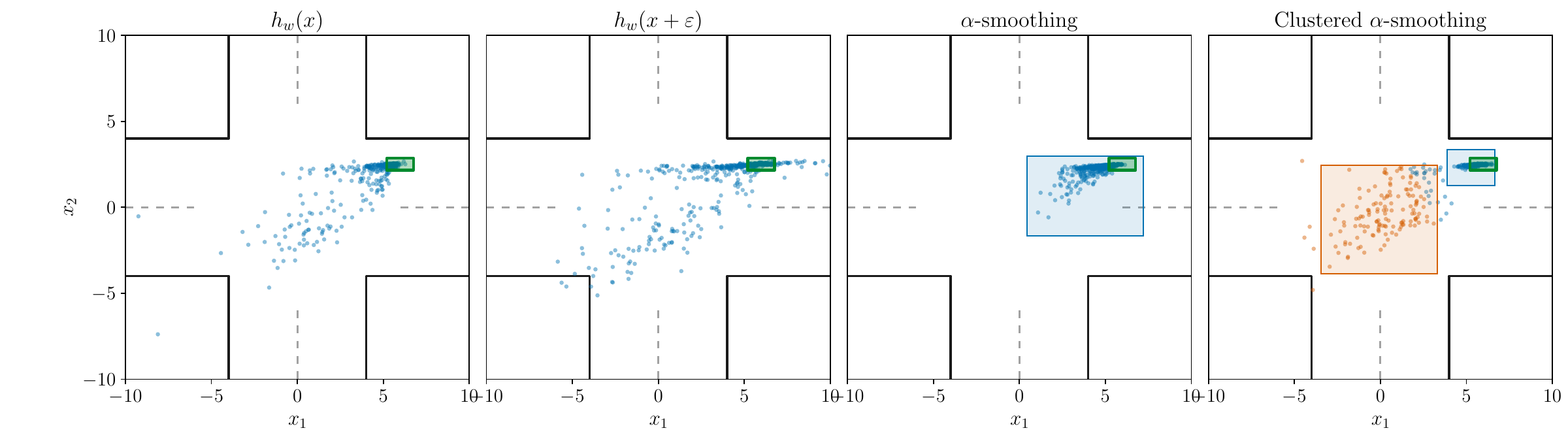}
    \vspace{-0.25cm}
    \caption{Given a particular trajectory input $x$ from the L-GAP dataset \cite{zgonnikov_should_2024}, we show \review{the last time step of trajectory samples} from $h_\randvar{w}(x)$, $h_\randvar{w}(x+\bm{\varepsilon})$, where $\bm{\varepsilon} \sim \sN(0, I)$, the $\alpha$-smoothing in \eqref{eq:def-alpha-smoothing}, and our clustered $\alpha$-smoothing. We also show \review{the projection to the last time step subspace of the sets} $\sR_m$ produced with Algorithm \ref{alg:sets-construction} for coverage $p=0.8$.}
    \label{fig:trajectory-prediction-samples-experiment}
    \vspace{-0.5cm}
\end{figure*}

\subsection{Quadrotor}\label{subsection:additional_quadrotor_results}

\review{As the goal of our method is to preserve the major modes of the original predictor while obtaining formal guarantees, we study in Figure~\ref{fig:quadrotor_wasserstein} the 2-Wasserstein distance between the base predictor $h_\mathbf{w}(x)$ and the smoothed predictors (Definition~\ref{def:alpha-smoothing} and \ref{def:clustered-alpha-smoothing}) for states $x$ sampled uniformly from across the state space $\mathcal{X}$. The figure shows that for uni-modal states, our approach performs comparably to $\alpha$-smoothing as expected by Remark~\ref{rmk:compare-clustered-smoother-definition-rekavandi}; the only difference arising if the clustering algorithm \review{predicts} more than one cluster. However, for states with a multi-modal output, $\alpha$-smoothing has a mean of 2.49, while our approach has a mean of 1.49, with the vast majority of its mass below 2.49.}

\begin{figure}
    \centering
    \includegraphics[width=\textwidth]{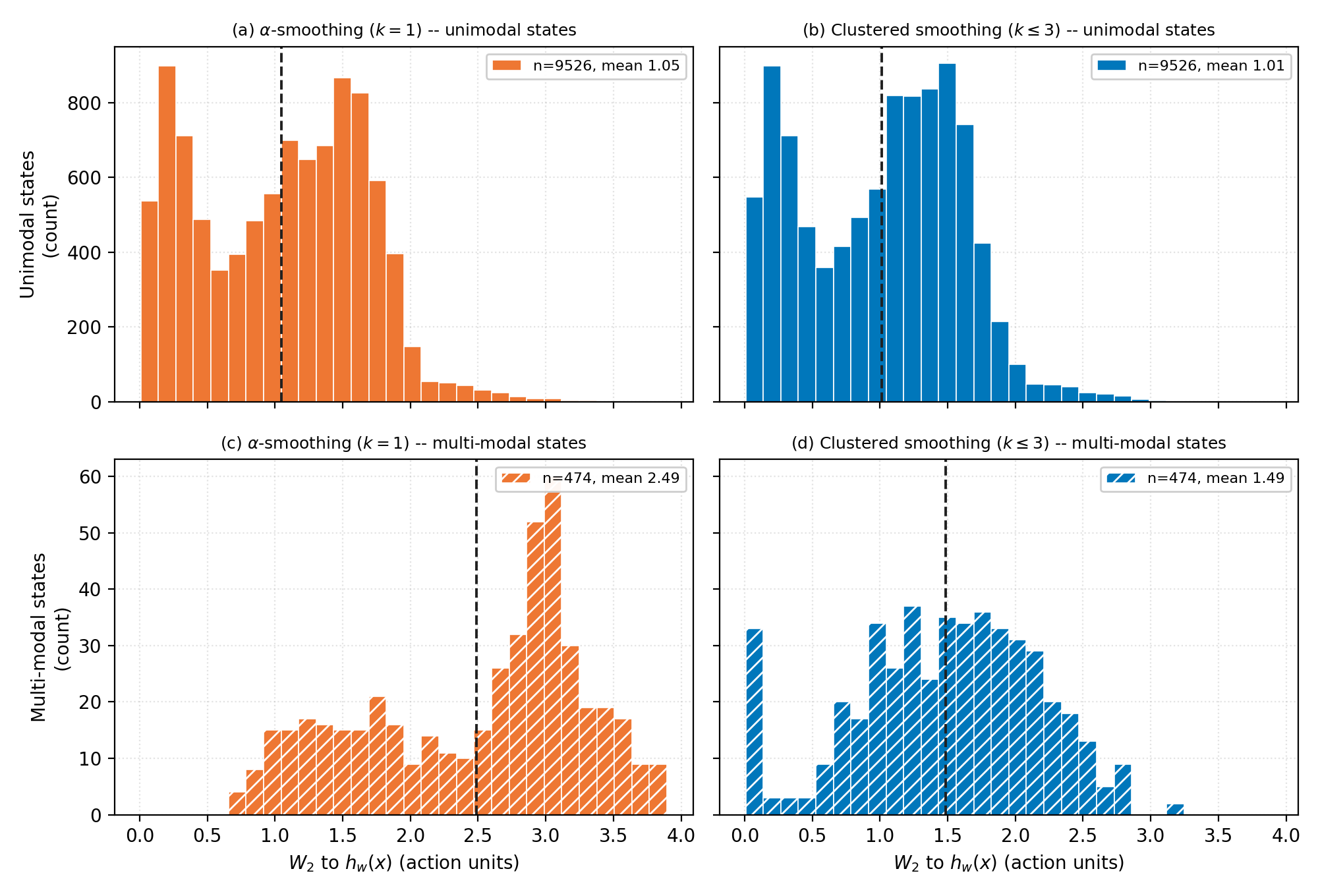}
    \caption{Distortion of the action distribution: $W_2$ from the base policy $h_w(x)$ to each smoothed policy over 10000 states drawn uniformly from the state space minus the obstacles. The states are sorted into uni-modal and multi-modal categories based on a threshold weight and distance between the modes of the base policy: the smallest mode of the two must have at least 20\% weight \emph{and} the distance between the two modes must be at least 1.5 (in output space) for a state to be classified as multi-modal.}
    \label{fig:quadrotor_wasserstein}
\end{figure}

\review{\paragraph{Adversarial attack.}
We have furthermore conducted experiments of the performance of the quadrotor under adversarial perturbations, to ensure that the smoothing does not degrade performance. For the remainder of this subsection, we first describe the design of the adversarial attack, immediately followed by the results.} 

\review{\textit{Threat model.} The attacker perturbs the \emph{observation}, not the physical state. At step $t$ the controller consumes $x_t + \delta_t$, where $x_t$ is the (normalized) state and $\|\delta_t\|_2 \le r$ is the disturbance, while the environment evolves from the uncorrupted $x_t$. The perturbation is recomputed every step, and each policy is attacked as a white-box model. Note that the randomized smoothing procedure can be interpreted as robustifying the controller against state uncertainty, e.g., arising from sensor noise.}

\review{\textit{Objective.} Let $\pi_{\mathrm{method}}(\cdot\mid x)$ be the action
distribution for $\mathrm{method} \in \{\mathrm{base},\alpha\mathrm{-smooth},\mathrm{clustered}\}$ on observation $x$, then the attacker solves the following constrained optimization problem}
\begin{equation}
  \min_{\|\delta\|_2 \le r}\;
  \mathbb{E}_{a \sim \pi_{\mathrm{method}}(\cdot\,\mid\, x + \delta)}\left[Q(x, a)\right],
  \qquad
  Q(x, a) = \mathbb{E}_{\mathbf{w}}\left[r(x, x'_\mathbf{w}) + \gamma V^{\pi}(x'_\mathbf{w})\right],
  \label{eq:uav-attack}
\end{equation}
\review{where $x'_\mathbf{w} = Ax + Ba + \mathbf{w}$ is the (random) next state, $\mathrm{r}(x, x') = \frac{1}{2}\left(\lVert x'_p \rVert_G - \lVert x_p \rVert_G \right) - 0.01$ is the progress reward, $x_p$ is the position component of $x$, $\lVert \cdot \rVert_G$ is the distance to the center of the goal region $G$, $\gamma$ is the RL discount factor, and $V^\pi$ is the critic from the PPO training. The inner expectation is computed by a product Gauss--Hermite rule with three nodes per noisy axis; positions are noise-free in a single step, so the random variable $\mathbf{w}$ only affects $V^\pi$.}

\review{\textit{Optimizer.} We run Projected Gradient Descent for $K = 40$ gradient-normalized steps of size $2.5r/K$, projecting onto the ball each step; the run is repeated from three starting points ($\delta = 0$ and two draws uniform in the ball) and we keep the best iterate over all of them. Each evaluation of the objective averages $M = 64$ realizations to estimate $\mathbb{E}_{a \sim \pi_{\mathrm{method}}(\cdot\,\mid\, x + \delta)}\left[Q(x, a)\right]$.}

\review{\textit{Results.}
We report the mean and standard deviation of the goal and collision rates for an attack of radius up to $0.05$ in Figure~\ref{fig:quadrotor_adversarial_env_and_rollout}. Figure~\ref{fig:quadrotor_adversarial_env_and_rollout} also shows the action distribution at the initial state to showcase how our method preserves the primary modes. Moreover, we show in Figure~\ref{fig:adversarial_attack_sweep} the outcomes and discounted returns as a function of the attack radius $r$ to show how our approach retains its performance until a larger attack radius than both the raw policy and $\alpha$-smoothing, thereby highlighting its robustness.}

\begin{figure}
    \centering
    \begin{subfigure}{0.3\textwidth}
        \centering
        \includegraphics[width=\textwidth]{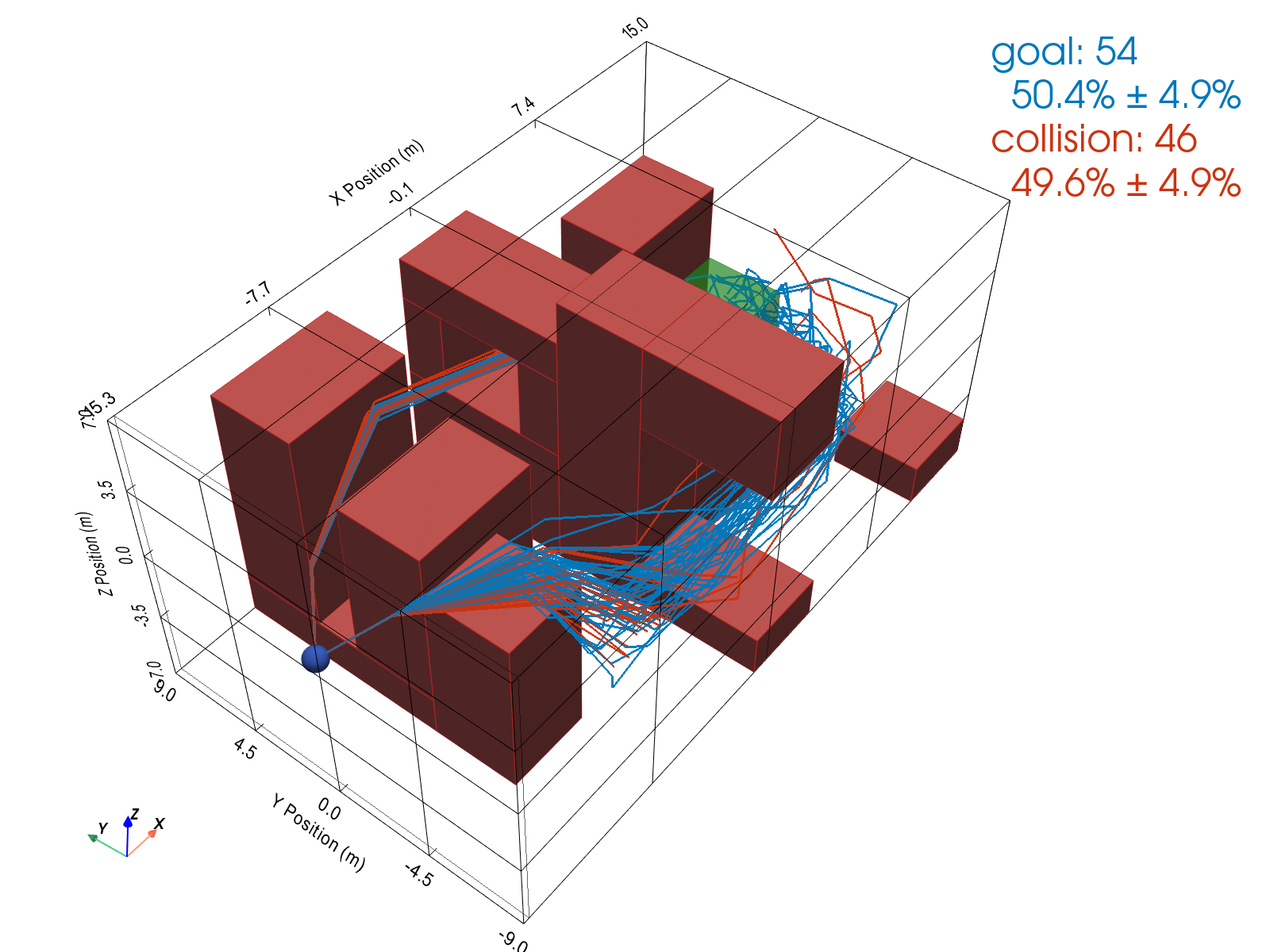}\\[0.3em]
        \includegraphics[width=0.85\textwidth]{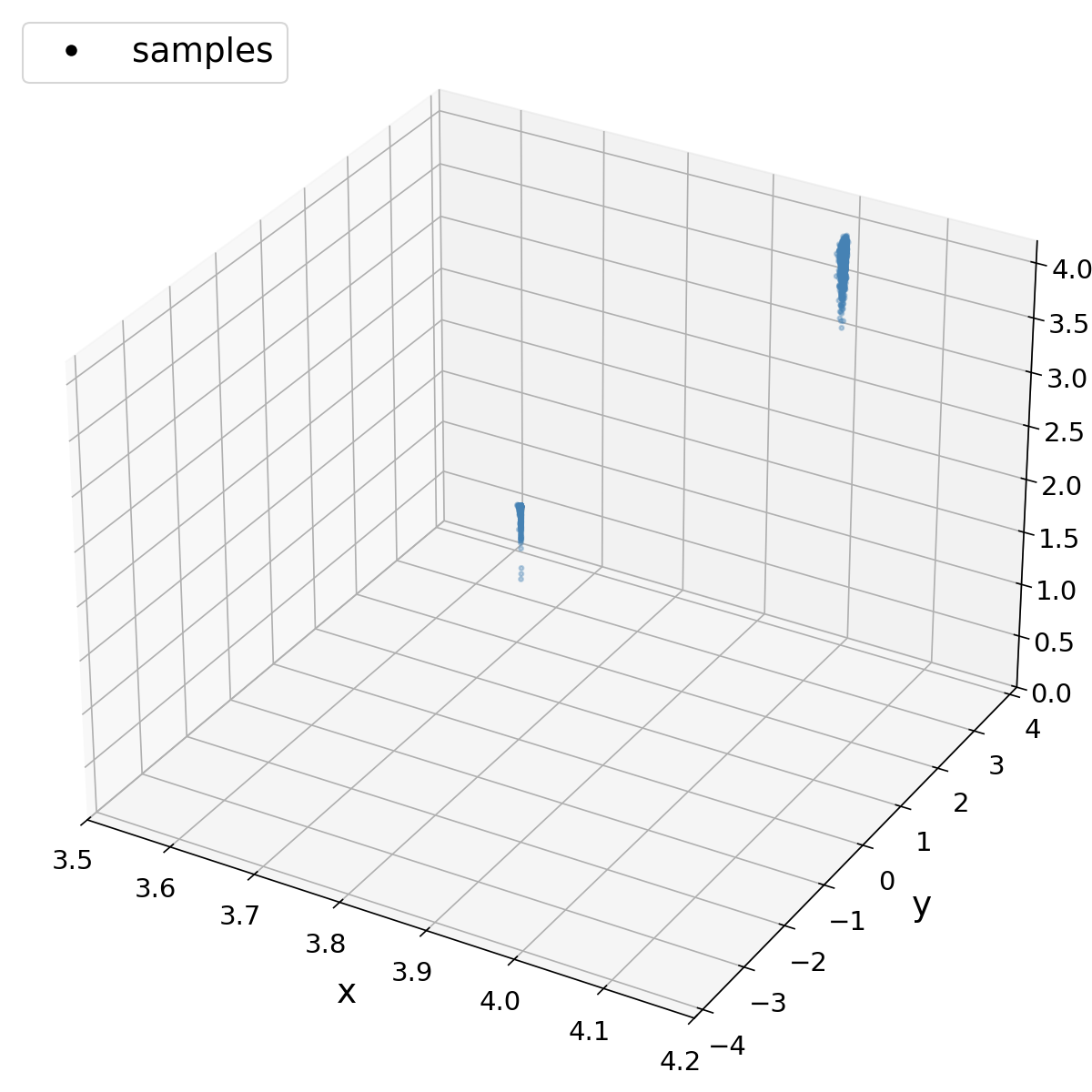}
        \caption{nominal policy}
    \end{subfigure}%
    \hfill
    \begin{subfigure}{0.3\textwidth}
        \centering
        \includegraphics[width=\textwidth]{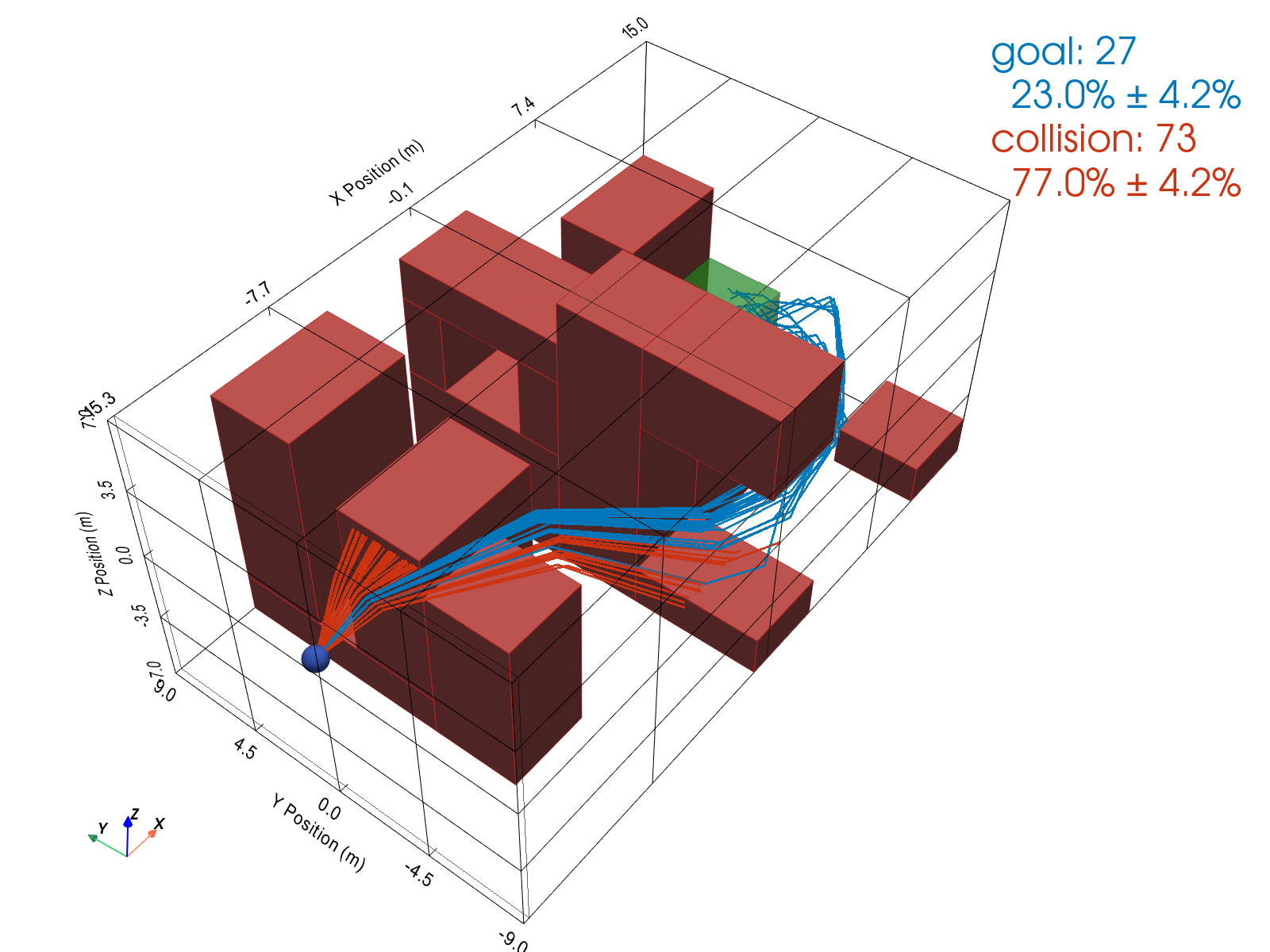}\\[0.3em]
        \includegraphics[width=0.85\textwidth]{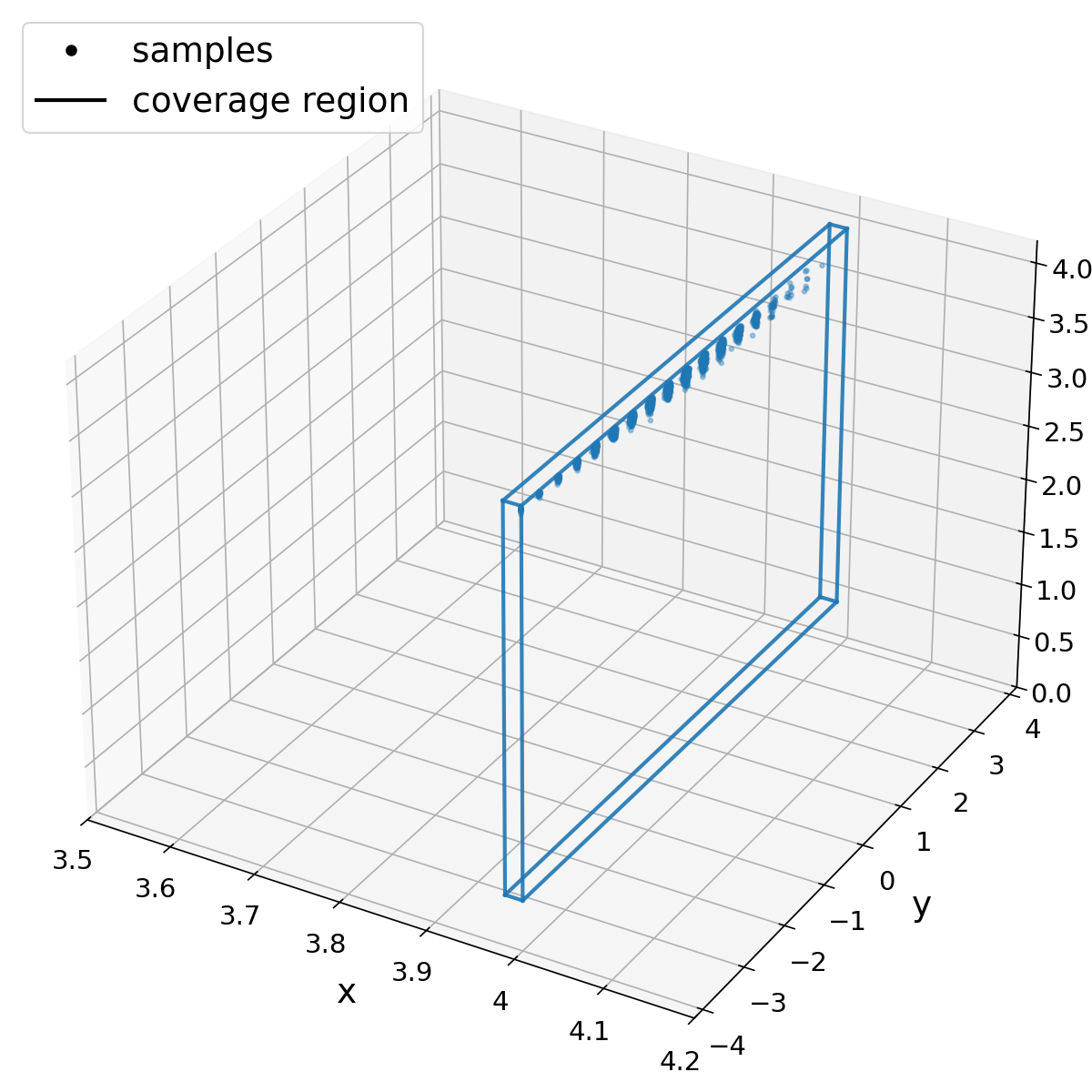}
        \caption{$\alpha$-smoothing \cite{rekavandi2024certified}}
    \end{subfigure}%
    \hfill
    \begin{subfigure}{0.3\textwidth}
        \centering
        \includegraphics[width=\textwidth]{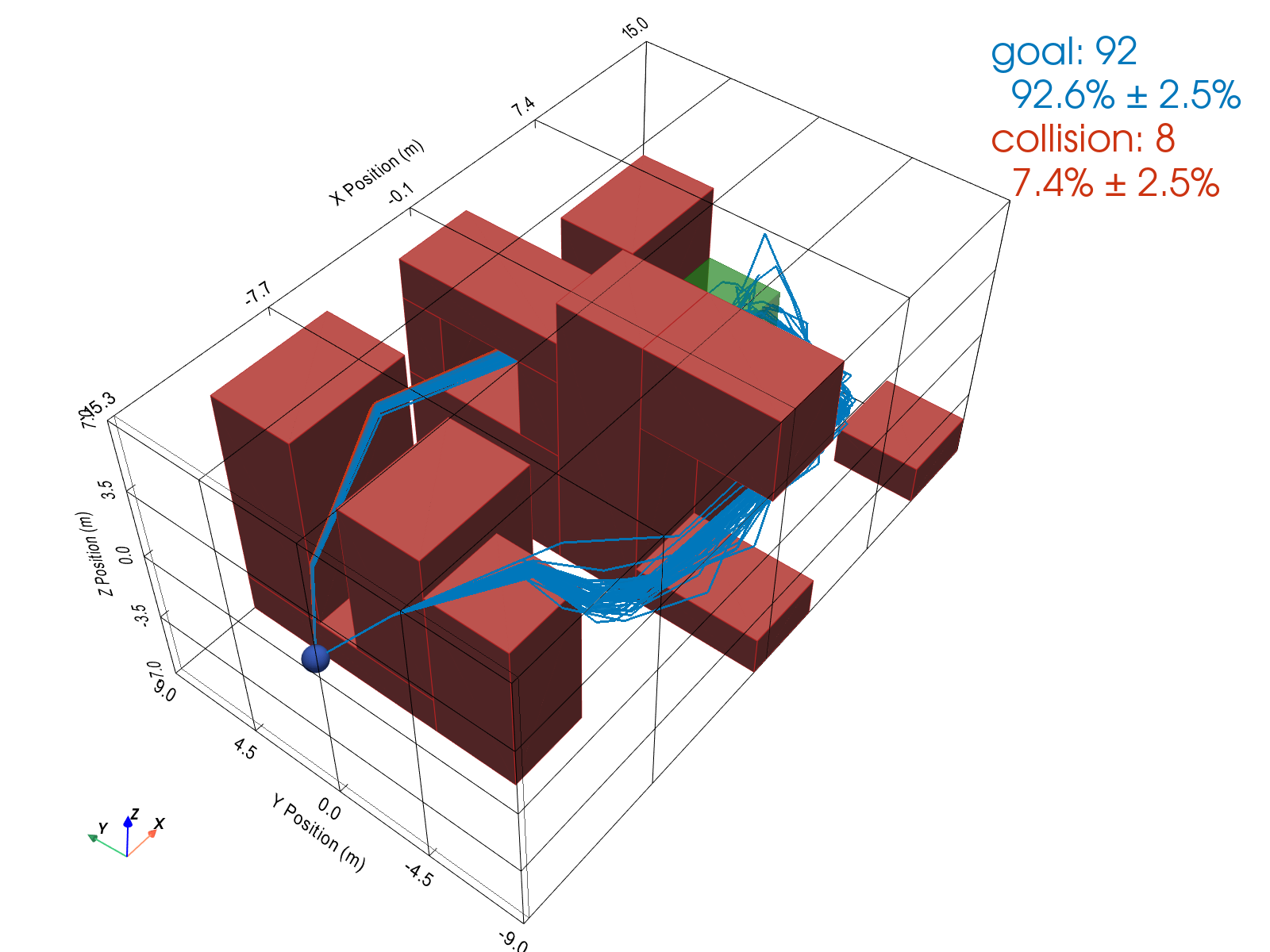}\\[0.3em]
        \includegraphics[width=0.85\textwidth]{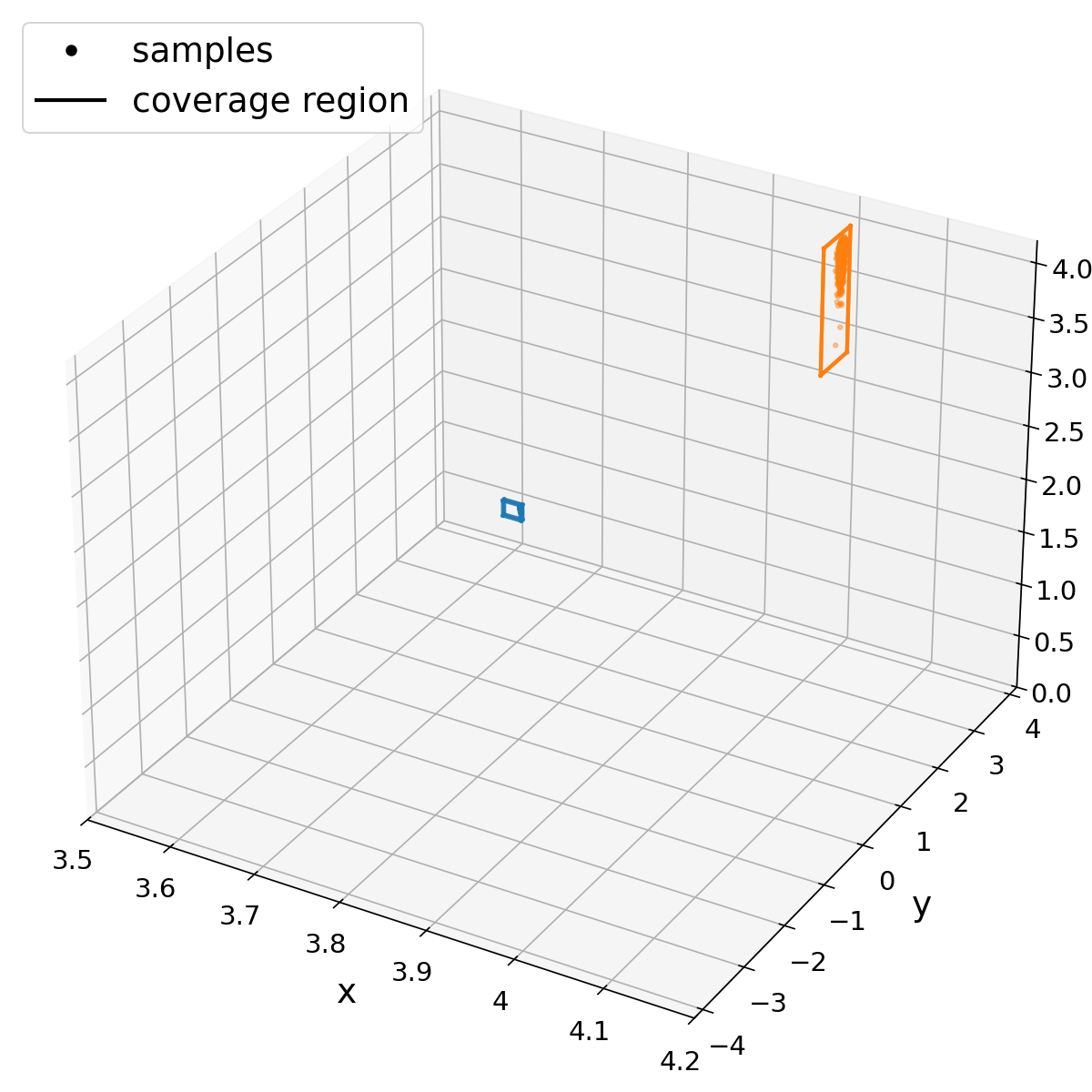}
        \caption{clustered $\alpha$-smoothing (ours)}
    \end{subfigure}
    \caption{Rollouts of the quadrotor when subject to the adversarial attack in Subsection~\ref{subsection:additional_quadrotor_results} with an attack budget of $r = 0.03$, which is equal to $0.6$ standard deviation of $\bm{\varepsilon}$. Similar to Figure~\ref{fig:quadrotor_env_and_rollout}, the top number shows the number of trajectories (out of 100) that reaches the goal of the depicted trajectories, while the number below is the mean and standard deviation of the goal rate over 100 simulations of 100 trajectories.}
    \label{fig:quadrotor_adversarial_env_and_rollout}
    \vspace{-0.5cm}
\end{figure} 

\begin{figure}
    \centering
    \includegraphics[width=\textwidth]{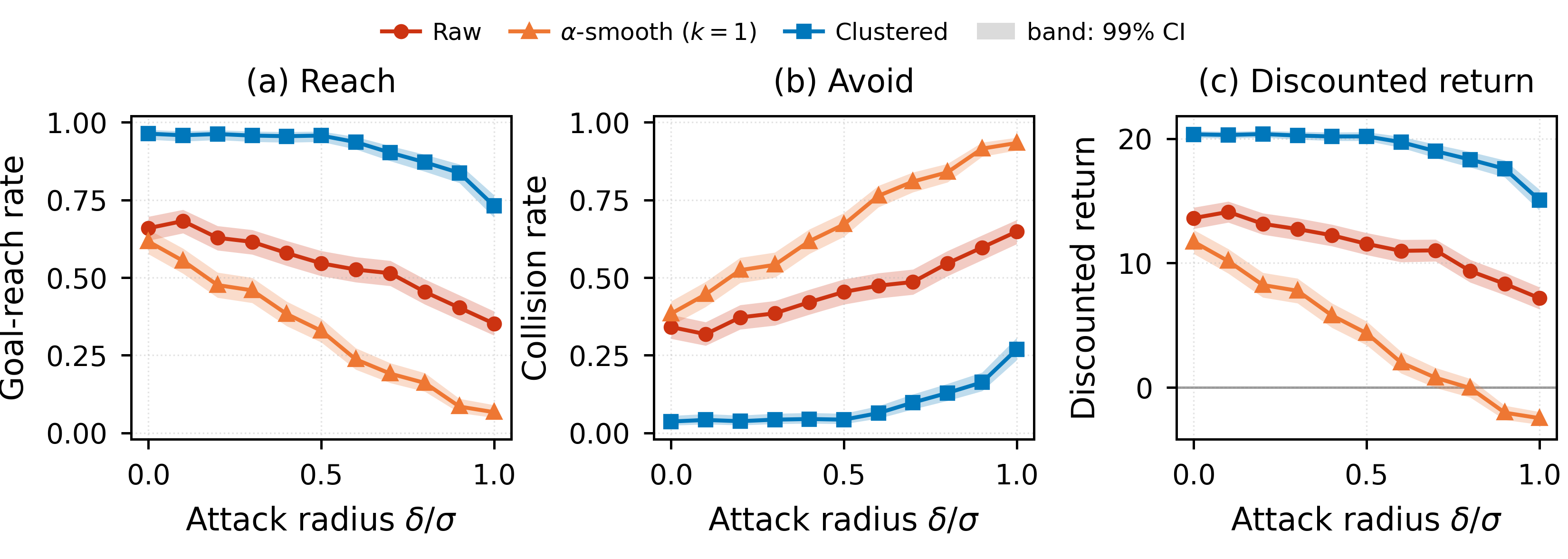}
    \caption{The simulation outcome and expected discounted return of the quadrotor under the adversarial attack in Subsection~\ref{subsection:additional_quadrotor_results} and starting the initial state $x_0 = (-14.8, 0, 0.7, 0, 0, 0)$. The $x$-axis represents the attack budget $r$ in coordinates normalized by the standard deviation of the random variable $\bm{\varepsilon}$. The mean and confidence intervals are calculated from 1000 trajectories at each radius level.}
    \label{fig:adversarial_attack_sweep}
\end{figure}

\section{Generalized Voronoi partition}\label{section:generalized_voronoi}

Given a collection of points $\bsC = \{ c_1,\dots, c_M \} \subset \sX$, the Voronoi partition $\tilde\bsV$ induced by $\bsC$ is defined\footnote{Although we introduce the Voronoi partition for the Euclidean distance, we must stress that it can be defined for any other metric.} as $\tilde\sV_i = \{ x \in \sX : \norm{x - c_i}_2 \leq \norm{x - c_j}_2, i\neq j \}$, i.e., the point $x$ belongs to the cell $\tilde\sV_i$ if it is closer to $c_i$ than to any other point $c_j$ in $\bsC$. The Generalized Voronoi partition extend this notion to sets, provided a shortest-distance measure from $x$ to a set $\sR \subseteq \sX$. We write
\begin{equation}\label{eq:distance-to-set-definition-appendix}
    d(x, \sR) = \inf_{r \in \sR}\, \lVert x - r \rVert_2
\end{equation}
for the (Euclidean) point-to-set distance, which equals $0$ if and only if $x \in \sR$ for closed $\sR$.

\begin{definition}[Generalized Voronoi partition]
For a collection $\bsR = \{\sR_1, \dots, \sR_M\}$ of pairwise disjoint closed sets in $\sX$, the generalized Voronoi cell associated with $\sR_i$ is given by
\[
    \sV_i = \big\{ x \in \sX : d(x, \sR_i) \mathrel{\reviewtheory{<}} d(x, \sR_j)
        \text{ for all } j \mathrel{\reviewtheory{<}} i, \;\; d(x, \sR_i) \le d(x, \sR_j) \text{ for all } j \mathrel{\reviewtheory{>}} i \big\},
\]
and $\bsV = \{\sV_i\}_{i=1}^{M}$ is the generalized Voronoi partition induced by $\bsR$. \reviewtheory{Ties are broken in favor of the lowest index, so that $\bsV$ is a partition of $\sX$ with $\sR_i \subset \sV_i$: for $x \in \sR_i$, $d(x, \sR_i) = 0 < d(x, \sR_j)$ for all $j \neq i$, as the $\sR_j$ are closed and disjoint from $\sR_i$.}
\end{definition}

Unlike $\tilde\bsV$, the cells $\sV_i$ in the generalized Voronoi partition $\bsV$ are generally \emph{not convex}, even when the sets $\sR_i$ are convex or simple shapes such as rectangles, as shown in Figure~\ref{fig:generalized_voronoi}. This fact, together with the combinatorial size of the partition that grows rapidly with both the number of sets $M$ and the dimension of the ambient space $\sX$, implies that computing the partition is almost certainly intractable. Indeed, even for the point-induced partition $\tilde\bsV$, the worst-case complexity in $\sX \subseteq \realNum^d$ is already $\Theta\left(M^{\lceil d/2 \rceil}\right)$.  When the sets $\sR_i$ are convex, the cell boundaries are $(d-1)$-dimensional algebraic surfaces, and for arbitrary compact sets the topology can be unbounded. Explicit construction of the partition is, therefore, intractable in $\realNum^d$ beyond toy examples. Crucially, however, our procedure does not require the explicit construction of the partition, but only identifying to which cell $\sV_i$ a given point $x \in \sX$ belongs, which reduces to computing $M$ point-to-set distances \eqref{eq:distance-to-set-definition-appendix}. This admits a trivial implementation when, for instance, the sets $\sR_i$ are axis-aligned hyperrectangles. 

\begin{figure}
    \centering
    \includegraphics[width=0.6\linewidth]{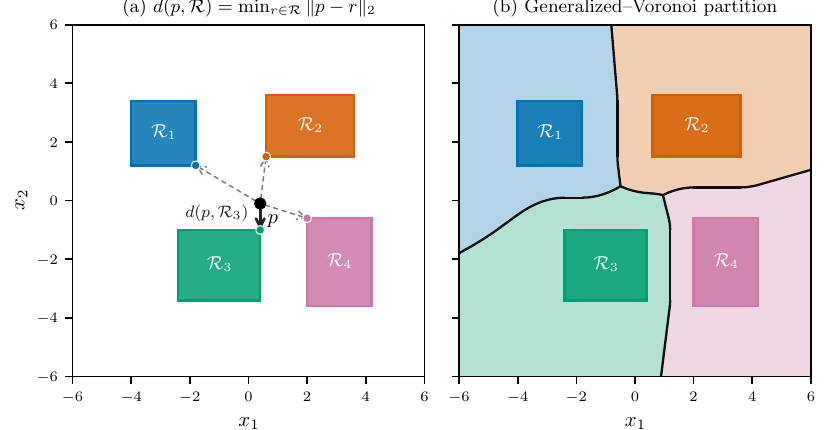}
    \caption{Generalized Voronoi partition of four disjoint axis-aligned rectangles $\sR_1, \dots, \sR_4 \subset \realNum^2$ under the Euclidean metric.
    (a) The distance from point $p$ to the closest point in each region $\sR_i$ is measured and $p$ is assigned $\sV_3$ as $\sR_3$ is the closest.
    (b) Each cell $\sV_i$ collects every point closer to $\sR_i$ than to any other site. Cell boundaries in the case of rectangles in 2D consist of straight segments and parabolic arcs and meet at \emph{Voronoi vertices} where three cells agree on
    the distance to a common point.}
    \label{fig:generalized_voronoi}
\end{figure}

\section{Quantization artifact in \texorpdfstring{$\alpha$}{alpha}-smoothing}\label{section:implicit_quantization}
As shown in Figure~\ref{fig:quadrotor_env_and_rollout}, $\alpha$-smoothing on bi-modal distributions unintuitively outputs a ``quantized'' smoothed prediction. The prediction concentrates on a small, near-uniformly spaced set of \emph{quantization levels} whose positions and weights are fully determined by the mixture and trimming parameters. We describe this artifact with the simplest non-trivial case, a 1-D two-component Gaussian mixture, which is enough to expose the mechanism. The same reasoning extends mode-by-mode to mixtures with more components or to higher-dimensional outputs (where each axis is quantized independently after coordinate-wise sorting).

\subsection{Setup}
Let the output distribution be the bi-modal Gaussian mixture
\begin{equation}\label{eq:gmm}
  p(y) = \pi_A \mathcal{N}(y \mid \mu_A, \sigma^2) + \pi_B \mathcal{N}(y \mid \mu_B, \sigma^2), \qquad \pi_A + \pi_B = 1,
\end{equation}
with $\mu_A < \mu_B$ and noise scale $\sigma$ small relative to the mode gap $\mu_B - \mu_A$. At a fixed smoothed input, the predictor draws $N$ samples $Y_1, \dots, Y_N \stackrel{\text{i.i.d.}}{\sim} p$, sorts them as $Y_{(1)} \le \dots \le Y_{(N)}$, and returns the symmetric $\alpha$-trimmed mean
\begin{equation}\label{eq:trimmed}
  \hat{\mu}_\alpha = \frac{1}{M}\sum_{i = L+1}^{N - L} Y_{(i)}, \qquad L = \lfloor \alpha N \rfloor, \quad M = N - 2L.
\end{equation}
Each Monte Carlo realization of the smoothed predictor is one independent draw of $\hat{\mu}_\alpha$. We call the set $\{Y_{(L+1)}, \dots, Y_{(N-L)}\}$ the trimmed core.

\subsection{Quantization in the noiseless limit}
Let $K$ be the (random) number of samples drawn from mode~$A$. By construction $K \sim \mathrm{Binom}(N, \pi_A)$. In the noiseless limit $\sigma \to 0$ the modes are perfectly separated, so after sorting the first $K$ values come from mode~$A$ and the remaining $N - K$ from mode~$B$. The trimmed core then contains $c_A(K) = \mathrm{clip}\left(K - L, 0, M\right)$ values from mode $A$ and $c_B(K) = M - c_A(K)$ from mode $B$. The trimmed mean then reduces to
\begin{equation}
  \hat{\mu}_\alpha^{\sigma=0}(K) = \frac{c_A(K)\mu_A + c_B(K)\mu_B}{M}.
\end{equation}
Since $c_A$ takes only integer values $0, 1, \dots, M$, the predictor collapses onto $M+1$ discrete \emph{quantization levels}
\begin{equation}\label{eq:levels-noiseless}
    \ell_j = \frac{j \mu_A + (M-j) \mu_B}{M}, \qquad j = 0, 1, \dots, M,
\end{equation}
spaced \emph{uniformly} between $\mu_A$ and $\mu_B$ at increments of \review{$(\mu_B - \mu_A)/M$}. Each level carries the binomial probability that $K$ takes a value mapping to it,
\begin{equation}
  \pi_j = \begin{cases}
     \Pr[K \le L] & j = 0,\\[2pt]
     \Pr[K = L + j] & j = 1, \dots, M-1,\\[2pt]
     \Pr[K \ge N - L] & j = M.
   \end{cases}
   \label{eq:level-weights}
\end{equation}
The two boundary levels absorb the tails of $\mathrm{Binom}(N, \pi_A)$
because any $K \le L$ saturates $c_A(K) = 0$ and any $K \ge N - L$ saturates
$c_A(K) = M$.

\paragraph{Addressing the failure mode of $\alpha$ smoothing.}
Increasing $N$ and decreasing $\alpha$ help alleviate the behavior by enlarging the trimmed core, and as $N \to \infty$ or $\alpha \to 0$, the distribution approaches a continuous distribution. However, both are problematic. Decreasing $\alpha$ means including more outliers, and increasing $N$ collapses the distribution to a tight peak around the mean, which is even more disjoint from the individual modes. 
Our clustered $\alpha$-smoothing avoids the artifact by summarizing each mode with its own trimmed mean and encoding the binomial split in the partition weights $\pi^{(m)}$ instead of in the level positions.

\subsection{Order-statistic correction at finite \texorpdfstring{$\sigma$}{sigma}}
\label{app:os-correction}

The noiseless scenario highlights the origin of the quantization behavior, but neglects a bias in the level positions for $\sigma > 0$. The mode-$A$ samples that survive trimming are not arbitrary mode-$A$ draws but specific order statistics of the $K$ mode-$A$ samples (those closer to mode-$B$ after sorting), and analogously for mode~$B$. Both contributions are biased \emph{toward} the gap between the modes.

Conditioning on $K = k$, the $i$-th overall order statistic is
\begin{equation}
  Y_{(i)} \mid K = k
   \;\stackrel{d}{=}\;
   \begin{cases}
     A_{(i:k)}, & i \le k,\\
     B_{(i-k:\,N-k)}, & i > k,
   \end{cases}
\end{equation}
where $A_{(j:n)}$ and $B_{(j:n)}$ denote the $j$-th order statistic among $n$ i.i.d.\ draws from the corresponding mode. Writing $\nu_{j:n} = \mathbb{E}[Z_{(j:n)}]$ for the expected $j$-th order statistic of $n$ standard normals, the conditional expected trimmed mean is
\begin{equation}
  \mathbb{E}\left[\hat{\mu}_\alpha \mid K = k\right] =  \frac{1}{M}\left[c_A(k)\mu_A + c_B(k)\mu_B + \sigma\left(S_A(k) + S_B(k)\right)
   \right],
   \label{eq:cond-mean}
\end{equation}
with order-statistic correction terms
\begin{align}
  S_A(k) &= \sum_{i = L+1}^{\min(k, N-L)} \nu_{i:k}
           \quad\text{(0 if $k \le L$),}\\[2pt]
  S_B(k) &= \sum_{i = \max(L+1, k+1)}^{N-L} \nu_{i-k:N-k}
           \quad\text{(0 if $k \ge N-L$).}
\end{align}
Because the expected order statistic is monotone in its rank, $S_A(k)$ is non-negative (we are averaging the \emph{upper} order statistics of the mode-$A$ sub-sample whenever $k > L$) and $S_B(k)$ is non-positive (lower order statistics of mode~$B$). Both shifts pull
$\mathbb{E}[\hat{\mu}_\alpha \mid K = k]$ inward, i.e., toward the middle of
the gap. The shifts are equal and opposite when $c_A(k) = c_B(k) = M/2$, so
the central level $\ell_{M/2}$ remains at $(\mu_A + \mu_B)/2$.

The marginal expected level for the $j$-th quantization bin is the binomial
weighted average of \eqref{eq:cond-mean} over the $k$ values mapping to that
bin (a single $k$ for interior bins, a tail of values for the two boundary
bins).

\subsection{Worked example: \texorpdfstring{$N = 30$, $\alpha = 0.4$, $\pi_A = 0.4$}{N=30, alpha=0.4, piA=0.4}}
With these parameters, the trimmed core is $M = 6$ samples, giving $M + 1 = 7$ quantization bins. The bins are nearly uniformly spaced; the order-statistic correction compresses the outer bins inward by roughly $0.22 \approx 0.7\sigma$ and leaves the center bin invariant. Most of the probability mass concentrates at the boundary level closest to $\mu_B$ because $\mathbb{E}[K] = N\pi_A = 12 = L$, placing the median of the
binomial split exactly at the lower trimming boundary.

Figure~\ref{fig:quantization} shows the underlying mixture (top) alongside the empirical histogram of $\hat{\mu}_\alpha$ over $20{,}000$ Monte Carlo realizations (bottom). The dashed lines mark the corrected levels $\bar{\ell}_j$ and labels report the binomial weights $\pi_j$ from \eqref{eq:level-weights}. Both line up tightly with the empirical peaks, confirming \eqref{eq:levels-noiseless}-\eqref{eq:cond-mean}.


\begin{figure}
  \centering
  \includegraphics[width=0.7\linewidth]{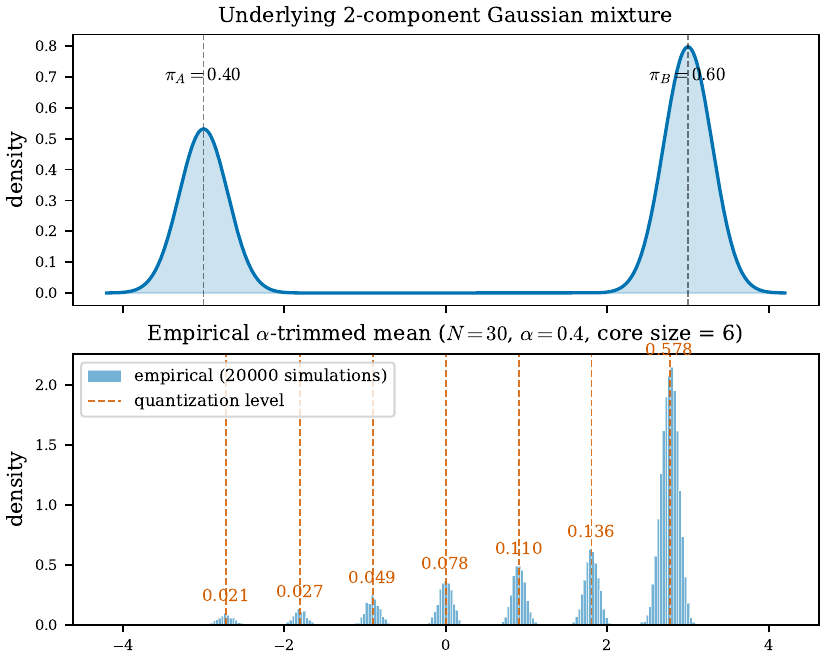}
  \caption{Quantization of the single-cluster $\alpha$-trimmed predictor on
  a bimodal output. \textbf{Top:} the underlying mixture
  $0.4\mathcal{N}(-3, 0.3^2) + 0.6\,\mathcal{N}(+3, 0.3^2)$.
  \textbf{Bottom:} histogram of $\hat{\mu}_\alpha$ for $N = 30$, $\alpha = 0.4$
  over $20{,}000$ realizations. Dashed lines are the predicted quantization
  levels $\bar{\ell}_j$ from \eqref{eq:cond-mean}; labels are the binomial
  weights $\pi_j$ from \eqref{eq:level-weights}.}
  \label{fig:quantization}
\end{figure}

\section{Control of the residual in Proposition \ref{prop:prob_bound_with_anchor_points}}\label{section:residual-analysis-prop-appendix}
One can achieve a desired residual level $\epsilon > 0$ by considering a uniform partition of the interval $[\ubar{p}_{\sV_l}, \bar{p}_{\sV_l}]$ into $K > \max_{l \in \sL} 1 + \frac{\lvert \sL \rvert}{\epsilon} (\bar{p}_{\sV_l}-\ubar{p}_{\sV_l}) \left[\frac{N}{1 - \bar{p}_{\sV_l}} + \frac{N}{\ubar{p}_{\sV_l} - \ubar{p}_{\sR_l}}\right]$ anchor points \reviewtheory{(including both endpoints, so that $h_{l,K} = \frac{\bar{p}_{\sV_l} - \ubar{p}_{\sV_l}}{2(K-1)}$)}. Then, it holds that $\sum_{l \in \sL} \left[\frac{N}{1 - \bar{p}_{\sV_l}} + \frac{N}{\ubar{p}_{\sV_l} - \ubar{p}_{\sR_l}}\right] h_{l, K} \leq \epsilon$.


\end{document}